\documentclass[a4paper, 11pt]{article}
\usepackage[margin=1.2in]{geometry}
\usepackage{times}
\usepackage{savesym}
\usepackage{amsmath,amssymb,amsfonts,amsthm}
\usepackage[numbers]{natbib}
\usepackage{multicol}
\usepackage{algorithm}
\usepackage[noend]{algpseudocode}
\usepackage{comment}
\usepackage{xspace}
\usepackage{scalerel}
\usepackage{xcolor}
\usepackage{subcaption}
\usepackage{multicol}
\usepackage{enumerate}
\usepackage{thmtools}
\usepackage{wrapfig}

\begin{document}
	
\title{The Critical Radius in Sampling-based Motion Planning}

\author{Kiril Solovey\footnote{Department of Aeronautics and Astronautics, Stanford University, CA~94305, US. email: \url{kirilsol@stanford.edu}}~~and  Michal Kleinbort\footnote{Blavatnik School of Computer Science, Tel Aviv University, Israel. email: \url{balasmic@post.tau.ac.il}}}

\maketitle
\thispagestyle{empty}
	
\def\frechet{Fr\'echet\xspace}

\newcommand{\cupdot}{\mathbin{\mathaccent\cdot\cup}}

\newcommand{\mtm}{\emph{multi-to-multi}\xspace}
\newcommand{\mts}{\emph{multi-to-single}\xspace}
\newcommand{\sts}{\emph{multi-to-single-restricted}\xspace}
\newcommand{\dtd}{\emph{single-to-single}\xspace}

\newcommand{\cte}{\emph{full-to-edge}\xspace}
\newcommand{\ctc}{\emph{full-to-full}\xspace}
\newcommand{\ete}{\emph{edge-to-edge}\xspace}

\newcommand{\AND}{{\sc and}\xspace}
\newcommand{\OR}{{\sc or}\xspace}

\newcommand{\ignore}[1]{}

\def\vor{\text{Vor}}

\def\P{\mathcal{P}} \def\C{\mathcal{C}} \def\H{\mathcal{H}}
\def\F{\mathcal{F}} \def\U{\mathcal{U}} \def\L{\mathcal{L}}
\def\O{\mathcal{O}} \def\I{\mathcal{I}} \def\S{\mathcal{S}}
\def\G{\mathcal{G}} \def\Q{\mathcal{Q}} \def\I{\mathcal{I}}
\def\T{\mathcal{T}} \def\L{\mathcal{L}} \def\N{\mathcal{N}}
\def\V{\mathcal{V}} \def\B{\mathcal{B}} \def\D{\mathcal{D}}
\def\W{\mathcal{W}} \def\R{\mathcal{R}} \def\M{\mathcal{M}}
\def\X{\mathcal{X}} \def\A{\mathcal{A}} \def\Y{\mathcal{Y}}
\def\L{\mathcal{L}}

\def\dS{\mathbb{S}} \def\dT{\mathbb{T}} \def\dC{\mathbb{C}}
\def\dG{\mathbb{G}} \def\dD{\mathbb{D}} \def\dV{\mathbb{V}}
\def\dH{\mathbb{H}} \def\dN{\mathbb{N}} \def\dE{\mathbb{E}}
\def\dR{\mathbb{R}} \def\dM{\mathbb{M}} \def\dm{\mathbb{m}}
\def\dB{\mathbb{B}} \def\dI{\mathbb{I}} \def\dM{\mathbb{M}}
\def\dZ{\mathbb{Z}}

\def\E{\mathbf{E}} 

\def\EE{\mathfrak{E}}

\def\eps{\varepsilon}

\def\limn{\lim_{n\rightarrow \infty}}

\def\obs{\mathrm{obs}}
\newcommand{\defeq}{%
  \mathrel{\vbox{\offinterlineskip\ialign{%
    \hfil##\hfil\cr
    $\scriptscriptstyle\triangle$\cr
    $=$\cr
}}}}
\def\Int{\mathrm{Int}}

\def\Reals{\mathbb{R}}
\def\Naturals{\mathbb{N}}
\renewcommand{\leq}{\leqslant}
\renewcommand{\geq}{\geqslant}
\newcommand{\compl}{\mathrm{Compl}}

\newcommand{\sig}{\text{sig}}

\newcommand{\sbs}{sampling-based\xspace}
\newcommand{\mr}{multi-robot\xspace}
\newcommand{\mpl}{motion planning\xspace}
\newcommand{\mrmp}{multi-robot motion planning\xspace}
\newcommand{\sr}{single-robot\xspace}
\newcommand{\cs}{configuration space\xspace}
\newcommand{\conf}{configuration\xspace}
\newcommand{\confs}{configurations\xspace}

\newcommand{\stl}{\textsc{Stl}\xspace}
\newcommand{\boost}{\textsc{Boost}\xspace}
\newcommand{\core}{\textsc{Core}\xspace}
\newcommand{\leda}{\textsc{Leda}\xspace}
\newcommand{\cgal}{\textsc{Cgal}\xspace}
\newcommand{\qt}{\textsc{Qt}\xspace}
\newcommand{\gmp}{\textsc{Gmp}\xspace}

\newcommand{\Cpp}{C\raise.08ex\hbox{\tt ++}\xspace}

\def\concept#1{\textsf{\it #1}}
\def\ccode#1{{\texttt{#1}}}

\newcommand{\ch}{\mathrm{ch}}
\newcommand{\pspace}{{\sc pspace}\xspace}
\newcommand{\threesum}{{\sc 3Sum}\xspace}
\newcommand{\np}{{\sc np}\xspace}
\newcommand{\degree}{\ensuremath{^\circ}}
\newcommand{\argmin}{\operatornamewithlimits{argmin}}

\newcommand{\Gdisk}{\G^\textup{disk}}
\newcommand{\Gbt}{\G^\textup{BT}}
\newcommand{\Gsoft}{\G^\textup{soft}}
\newcommand{\Gnear}{\G^\textup{near}}
\newcommand{\Gembed}{\G^\textup{embed}}

\newcommand{\dist}{\textup{dist}}

\newcommand{\Cfree}{\C_{\textup{free}}}
\newcommand{\Cforb}{\C_{\textup{forb}}}

\newtheorem{lemma}{Lemma}
\newtheorem{theorem}{Theorem}
\newtheorem{corollary}{Corollary}
\newtheorem{claim}{Claim}
\newtheorem{proposition}{Proposition}

\theoremstyle{definition}
\newtheorem{definition}{Definition}
\newtheorem{remark}{Remark}
\theoremstyle{plain}
\newtheorem{observation}{Observation}


\newtheorem{lemma2}{Lemma}
\newtheorem{theorem2}{Theorem}
\newtheorem{corollary2}{Corollary}
\newtheorem{claim2}{Claim}
\newtheorem{proposition2}{Proposition}

\theoremstyle{definition}
\newtheorem{definition2}{Definition}

\def\len{c_\ell}
\def\bot{c_b}

\def\lenopt{\len^*}
\def\botopt{\bot^*}

\def\Im{\textup{Im}}

\def\rfunc{\left(\frac{\log n}{n}\right)^{1/d}}
\def\rfuncs{\left(\frac{\log n}{n}\right)^{1/d}}
\def\cfunc{\sqrt{\frac{\log n}{\log\log n}}}
\def\rtrs{\gamma\rfunc}
\def\ctrs{2\cfunc}
\def\aconn{\A_\textup{conn}}
\def\abd{\A_\textup{str}}
\def\aspan{\A_\textup{span}}
\def\aopt{\A_\textup{opt}}
\def\ao{\A_\textup{ao}}
\def\acfo{\A_\textup{acfo}}
\def\binomial{\textup{Binomial}}
\def\twin{\textup{twin}}

\def\aas{a.a.s.\xspace}
\def\0{\bm{0}}

\def\distU#1{\|#1\|_{\G_n}^U}
\def\distW#1{\|#1\|_{\G_n}^W}

\def\tooth{\scalerel*{\includegraphics{./../fig/tooth}}{b}}

\makeatletter
\def\thmhead@plain#1#2#3{%
  \thmname{#1}\thmnumber{\@ifnotempty{#1}{ }\@upn{#2}}%
  \thmnote{ {\the\thm@notefont#3}}}
\let\thmhead\thmhead@plain
\makeatother

\def\todo#1{\textcolor{blue}{\textbf{TODO:} #1}}
\def\kiril#1{\textcolor{blue}{\textbf{KIRIL:} #1}}
\def\new#1{\textcolor{magenta}{#1}}
\def\old#1{\textcolor{red}{#1}}
\def\trim#1{\textcolor{red}{#1}}

\def\removed#1{\textcolor{green}{#1}}

\def\dx{\,\mathrm{d}x}
\def\dy{\,\mathrm{d}y}
\def\drho{\,\mathrm{d}\rho}

\newcommand{\prm}{{\tt PRM}\xspace}
\newcommand{\prmstar}{{\tt PRM}$^*$\xspace}
\newcommand{\rrt}{{\tt RRT}\xspace}
\newcommand{\est}{{\tt EST}\xspace}
\newcommand{\rrtstar}{{\tt RRT}$^*$\xspace}
\newcommand{\rrg}{{\tt RRG}\xspace}
\newcommand{\btt}{{\tt BTT}\xspace}

\newcommand{\lbtrrt}{{\tt LBT-RRT}\xspace}
\newcommand{\mplb}{{\tt MPLB}\xspace}
\newcommand{\spars}{{\tt SPARS2}\xspace}
\newcommand{\rsec}{{\tt RSEC}\xspace}
\newcommand{\bfmt}{{\tt BFMT}$^*$\xspace}

\newcommand{\fmt}{{\tt FMT}$^*$\xspace}
\newcommand{\mstar}{{\tt M}$^*$\xspace}
\newcommand{\drrt}{{\tt dRRT}\xspace}
\newcommand{\drrtstar}{{\tt dRRT}$^*$\xspace}


\begin{abstract}
  We develop a new analysis of sampling-based motion planning in
  Euclidean space with uniform random sampling, which significantly
  improves upon the celebrated result of Karaman and Frazzoli (2011)
  and subsequent work. Particularly, we prove the existence of a
  critical connection radius proportional to ${\Theta(n^{-1/d})}$ for
  $n$ samples and ${d}$ dimensions: Below this value the planner is
  guaranteed to fail (similarly shown by the aforementioned work,
  ibid.). More importantly, for larger radius values the planner is
  asymptotically (near-)optimal. Furthermore, our analysis yields an
  explicit lower bound of ${1-O( n^{-1})}$ on the probability of
  success.  A practical implication of our work is that asymptotic
  (near-)optimality is achieved when each sample is connected to only
  ${\Theta(1)}$ neighbors. This is in stark contrast to previous work
  which requires ${\Theta(\log n)}$ connections, that are induced by a
  radius of order ${\left(\frac{\log n}{n}\right)^{1/d}}$.  Our
  analysis is not restricted to \prm and applies to a variety of
  ``\prm-based'' planners, including \rrg, \fmt and \btt.  Continuum
  percolation plays an important role in our proofs. Lastly, we
  develop similar theory for all the aforementioned planners when
  constructed with deterministic samples, which are then sparsified in
  a randomized fashion. We believe that this new model, and its
  analysis, is interesting in its own right.
\end{abstract}

		
\maketitle
	
\section{Introduction}\label{sec:introduction}
Motion planning is a fundamental problem in robotics concerned with
allowing autonomous robots to efficiently navigate in environments
cluttered with obstacles. Although motion planning has originated as a
strictly theoretical problem in computer science~\citep{HSS16},
nowadays it is applied in various fields.  Notably, motion planning
arises in coordination of multiple autonomous
vehicles~\citep{FraPav15}, steering surgical needles~\citep{BayAlt17},
and planning trajectories of spacecrafts in orbit~\citep{StaETAL16}, to
name just a few examples.
	
Motion planning is notoriously challenging from a computational
perspective due to the continuous and high-dimensional search space it
induces, which accounts for the structure of the robot, the physical
constraints that it needs to satisfy, and the environment in which it
operates.  Low-dimensional instances of the problem can
  typically be solved by efficient and complete
  algorithms~\citep{HKS16}, and the same applies to several
  high-dimensional instances (see, e.g.,~\cite{AdlETAL15,
    SolETAL15}). However, in general, motion planning is
  computationally intractable (see, e.g.,~\cite{Can88, Rei79,
    SolHal16j}).
	
Nowadays the majority of practical approaches for motion planning
capture the connectivity of the free
space
 ~by sampling (typically in a randomized fashion) configurations
 and connecting nearby samples, to form a graph data
 structure. Although such \emph{sampling-based planners} are
 inherently incomplete, i.e., cannot detect situations in which a
 solution (collision-free path) does not exist, most have the
 desired property of being able to find a solution
 \emph{eventually}, if one exists. That is, a planner is
 \emph{probabilistically complete} (PC) if the probability of
 finding a solution tends to $1$ as the number of samples $n$ tends
 to infinity. Moreover, some recent sampling-based
 techniques are also guaranteed to return
 high-quality solutions
 that tend to the optimum as $n$ diverges---a property called
 \emph{asymptotic optimality} (AO). Quality can be measured in terms of energy,
 	length of the plan, clearance from obstacles, etc.
	
An important attribute of sampling-based planners, which dictates both
the running time and the quality of the returned solution, is the
number of neighbors considered for connection for each added
sample. In many techniques this number is directly affected by a
connection radius $r_n$: Decreasing $r_n$ reduces the number of
neighbors. This in turn reduces the running time of the planner for a
given number of samples $n$, but may also reduce the quality of the
solution or its availability altogether. Thus, it is desirable to come
up with a radius $r_n$ that is small, but not to the extent that the
planner loses its favorable properties of PC and AO.
	
\subsection{Contribution}
We develop a new analysis of \prm~\citep{KavETAL96} for uniform
random sampling in Euclidean space, which relies on a novel
connection between sampling-based planners and \emph{continuum
	percolation} (see, e.g.,~\cite{FraMee08}).  Our analysis is
tight and proves the existence of a \emph{critical connection}
radius $r^*_n=\gamma^*n^{-1/d}$, where $\gamma^*>0$ is a
constant (in particular, $0.4\leq \gamma^*\leq 0.6$ for all $d\geq 2$),
and $d\geq 2$ is the dimension: If $r_n<r^*_n$ then \prm is
guaranteed to fail, where $d$ is the dimension. Above the
threshold, i.e., when $r_n>r^*_n$, \prm is AO for the bottleneck
cost, and \emph{asymptotically near optimal}\footnote{AnO means
	that the cost of the solution tends to at most a constant factor
	times the optimum, compared with AO in which this constant is
	equal to one.}
  (AnO) with respect to the path-length
cost. Furthermore, our analysis yields concrete bounds on the
probability of success, 
which is
lower-bounded by $1-O(n^{-1})$. Notice that this bound is
comparable to the one obtained in~\cite{StaETAL15} (see
Section~\ref{sec:related}) although we show this for a much
smaller radius.
	
Our analysis is not restricted to \prm and applies to a variety of
planners that maintain \prm-like roadmaps, explicitly or
implicitly. For instance, when $r_n$ is above the threshold,
\fmt~\citep{JSCP15} is AnO with respect to the path-length cost, while
\btt~\citep{SolHal17} is AO with respect to the bottleneck cost. \rrg~\citep{KF11}
behaves similarly for the two cost functions. Our results are also
applicable to \emph{multi-robot} motion planners such as the recently
introduced \drrtstar~\citep{DobETAL17}, and \mstar~\citep{WagCho15} when
applied to a continuous domain. See Figure~\ref{fig:prm} for
additional \prm-based planners to which our analysis is applicable,
and which are mentioned further on.
	
\begin{figure}
  \centering
  \includegraphics[width=0.5\columnwidth]{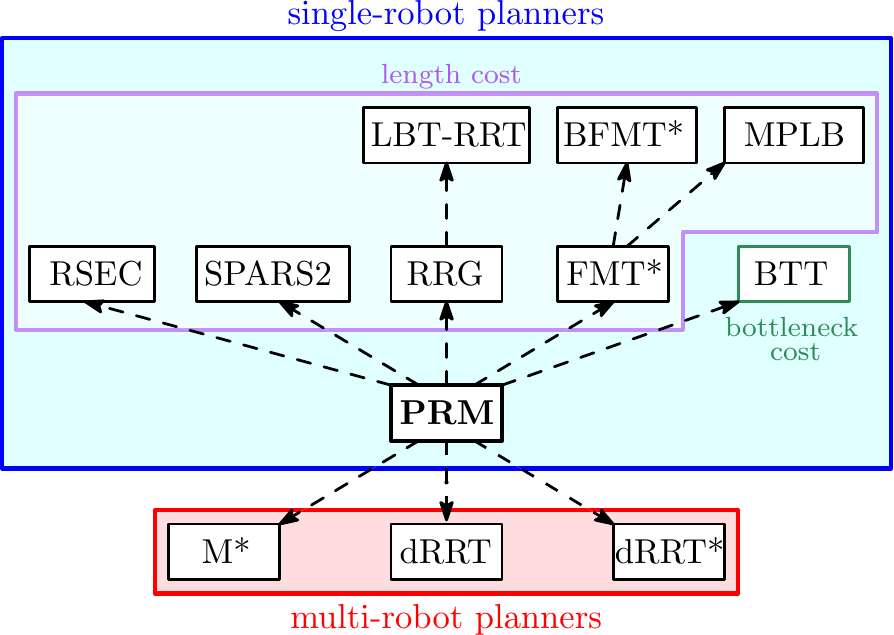}
\caption{The \prm dynasty. Single robot and multi-robot planners, are bounded into the blue and red frames, respectively. Planners inside the magenta frame aim to minimize the path-length cost, whereas \btt is designed for bottleneck cost, and \rrg works for both costs (see more information below). Roughly speaking, arrow from ``$A$'' to ``$B$'' indicates that theoretical properties of  ``$A$'' extend to ``$B$'', or that the latter  maintains ``$A$'' as a substructure. 
  \label{fig:prm}}
\end{figure}
	
A practical implication of our work is that AO (or AnO), under the
regime of uniform random sampling, can be achieved even when every
sample is connected to $\Theta(1)$ neighbors. This is in stark
contrast with previous work, e.g., \cite{JSCP15, KF11, SolHal16,
  SolETAL16}, which provided a rough estimate of this number, that is
proportional to $O(\log n)$.  Interestingly, our $\Theta(1)$ bound for
uniform samples is comparable to the best known bound of for
deterministic samples~\citep{LucETAL17}.
	
Lastly, we also study the asymptotic properties of the aforementioned
planners when constructed using a deterministic point process, rather
than uniform random sampling which we have mentioned so far. We
introduce an analysis which shows that the graphs constructed using
such deterministic samples can be sparsified in a randomized fashion
while maintaining AO or AnO of the planners. We term this regime as
\emph{semi-deterministic sampling}.
	
\subsection{Organization}
In Section~\ref{sec:related} we discuss related work. Then we proceed
to basic definitions and problem statement in
Section~\ref{sec:preliminaries}. In that section we also include a
precise description of the robotic system our theory is developed
for. Our central contribution (Theorem~\ref{thm:prm}), which states
the existence of a critical radius $r^*_n$ with respect to \prm, is
presented in Section~\ref{sec:prm}. This section also contains an
outline of the proof. In Section~\ref{sec:elements} we lay the
foundations of the proof and discuss important aspects of continuum
percolation, that would later be employed in the main proof, which
appears in Section~\ref{sec:proof}. In
Section~\ref{sup:other_planners} we extend the analysis to \fmt, \btt,
\rrg and discuss the implications of our results to the multi-robot
planners \drrtstar and \mstar. 
In
Section~\ref{sec:experimental_results} we present experimental work,
which validates our theory. We conclude this paper with directions for
future work (Section~\ref{sec:future}).
    
In the appendix we include proofs that were omitted from the main
document and a table with values of
$\gamma^*$. We also present in the appendix a new analysis  for the aforementioned planners when
constructed with deterministic samples, which are then sparsified in a
randomized fashion. We believe that this new model, and its analysis,
is interesting in its own right.
	
\section{Related work}\label{sec:related}
This section is devoted to a literature review of sampling-based
planners with emphasis on their theoretical analysis.
Sampling-based techniques were first described in the mid 90's and
several prominent examples of that time, that are still used
extensively today, include \prm~\citep{KavETAL96},
\rrt~\citep{LaVKuf01}, and \est~\citep{HLM99}. Those planners are
also known to be PC (see \cite{KavrakiETAL98,LadKav04,KSKBH18,HLM99}).
	
More recent work has been concerned with the quality of the returned
solution.  \cite{NecETAL10} were one of the first to address
this matter in a mathematically-rigorous manner. This work proved that
\rrt can produce arbitrarily-bad (long) paths with non-negligible
probability.
	
The influential work of \cite{KF11} laid the theoretical foundations for
analyzing quality in sampling-based planning. The authors introduced a
set of techniques to prove AO. Using their framework, they showed that
the following algorithms are AO: \prmstar, which is a special case of
\prm (throughout this work we will refer to the more general
  algorithm \prm rather than \prmstar)  with a specific value of the
connection radius $r_n$; an AO variant of \rrt~\cite{LaVKuf01} termed
\rrtstar; \rrg, which can be viewed as a combination between \rrt and
\prm.  The analysis in~\cite{KF11} establishes that
$r_n=\Theta\left((\log n/ n)^{1/d}\right)$ guarantees AO, where the
configuration space of the robot is assumed to be $[0,1]^d$. This
indicates that the expected number of neighbors used per vertex should
be $O(\log n)$.  The authors also proved that for sufficiently-small
radii of order $O(n^{-1/d})$ the planner is guaranteed to fail
(asymptotically) in finding any solution.
	
Following the breakthrough of~\cite{KF11}, other AO planners have
emerged (see e.g.,~\cite{APD11, AT13,GSB15}). \cite{JSCP15}
introduced \fmt, which is a single-query planner that traverses an
implicitly-represented \prm graph, and is comparable in performance to
\rrtstar. The authors refined the proof technique of \cite{KF11},
which allowed them to slightly reduce the connection radius $r_n$
necessary to \fmt and \prm to achieve AO. We do mention that here
again $r_n=\Theta\left((\log n/ n)^{1/d}\right)$. \bfmt, which is a
bidirectional version of \fmt, was introduced by \cite{StaETAL15}. In
this paper the authors also proved that the success rate of \prm,
\fmt, \bfmt can be lower bounded by
$1-O\left(n^{-\eta/d}\log^{-1/d}n\right)$, where $\eta>0$ is a tuning
parameter. In this context, we also mention the work of~\cite{DobETAL15}, which bounds the success rate with an expression
that depends on the amount of deviation from the optimum.
	
A recent work by~\cite{SolETAL16} developed a different method for
analyzing sampling-based planners. It exploits a connection with
\emph{random geometric graphs} (RGGs), which have been extensively
studied (see, e.g.,~\cite{Pen03}). Their work shows that one can
slightly reduce the \prm and \fmt radius obtained in~\cite{JSCP15}.
Furthermore, the connection with RGGs yields additional analyses of
different extensions of \prm, which have not been analyzed before in a
mathematically-rigorous setting.
	
A number of methods have been developed to reduce the running time or
space requirements of existing planners by relaxing AO constraints to
AnO. For instance, \lbtrrt~\citep{SH16} interpolates between the quick
\rrt and the AO yet slower \rrg, according to a predefined
parameter. \mplb~\citep{SH15} can be viewed as a relaxation of
\fmt. \spars~\citep{DB14} and \rsec~\citep{SalETAL14} perform a
sparsification of \prm in an online or offline fashion, respectively,
to reduce the space footprint of the produced roadmap.

\subsection{Extensions}
The aforementioned papers deal mainly with the cost function of
\emph{path length}. Two recent works~\citep{SolHal16, SolHal17}
considered the \emph{bottleneck-pathfinding} problem in a
sampling-based setting and introduced the \btt algorithm, which
traverses an implicitly-represented \prm graph. The bottleneck-cost
function, which arises for instance in high-clearance multi-robot
motion~\citep{SolHal17} and Fr\'{e}chet matching between
curves~\citep{HolETAL17}, is defined as follows: Every robot
configuration $x$ is paired with a value $\M(x)$, and the cost of a
path is the maximum value of $\M$ along any configuration on the
path. It was shown \citep{SolHal16, SolHal17} that \btt is AO, with
respect to bottleneck cost, for the reduced connection radius that was
obtained by~\cite{SolETAL16}.
	
The results reported until this point have dealt exclusively with
holonomic robotic systems. Two recent papers by \cite{SchETAL15,
  SchETAL15b} develop the theoretical foundations of \prm and
\fmt-flavored planners when applied to robots having differential
constraints. \cite{LiETAL16} develop an AO algorithm that
does not require a steering function, as \prm for instance does. Interestingly, the authors of the last paper also obtain bounds on the rate of convergence of their algorithm.

\section{Preliminaries}\label{sec:preliminaries}
We provide several basic definitions that will be used throughout the
paper. Given two points $x,y\in \dR^d$, denote by $\|x-y\|$ the
standard Euclidean distance. Denote by $\B_{r}(x)$ the $d$-dimensional
ball of radius $r>0$ centered at $x\in \dR^d$ and
$\B_{r}(\Gamma) = \bigcup_{x \in \Gamma}\B_{r}(x)$ for any
$\Gamma \subseteq \dR^d$.  Similarly, given a curve
$\pi:[0,1]\rightarrow \dR^d$ define
$\B_r(\pi)=\bigcup_{\tau\in[0,1]}\B_r(\pi(\tau))$. Given a subset
$D\subset \dR^d$ we denote by $|D|$ its Lebesgue measure. All
logarithms are at base~$e$. 
	
\subsection{Motion planning}
Denote by $\C$ the configuration space of the robot, and by
$\F\subseteq \C$ the free space, i.e., the set of all collision free
configurations. Though our proofs may be extended to more complex
robotic systems (see discussion in Section~\ref{sec:future}), in this
work we investigate the geometric (holonomic) setting of the problem
in which no constraints are imposed on the motion of the
robot. Additionally, we assume that $\C$ is some subset of the
Euclidean space. In particular, $\C=[0,1]^d\subset\dR^d$ for some
fixed $d\geq 2$. We also assume that for any two configurations
$x,x'\in \C$ the robot is capable of following precisely the
straight-line path from $x$ to $x'$.

Given start and target configurations $s,t\in \F$, the problem
consists of finding a continuous path (curve)
$\pi:[0,1]\rightarrow \F$ such that $\pi(0)=s,\pi(1)=t$. That is, the
robot starts its motion along $\pi$ on $s$, and ends in $t$, while
remaining collision free. An instance of the problem is defined for a
given $(\F,s,t)$, where $s,t\in \F$.
		
\subsection{Cost function}
It is usually desirable to obtain paths that minimize a given
criterion. In this paper we consider the following two cost functions.

\begin{definition}\label{def:length}
  Given a path $\sigma$, its \emph{length} is
  \[\len(\sigma)=\sup_{n\in \dN_+, 0=\tau_1\leq \ldots \leq \tau_n
    =1}\sum_{i=2}^n\|\pi(\tau_i)-\pi(\tau_{i-1})\|.\]
\end{definition}
	
\begin{definition}
  Given a path $\sigma$, and a \emph{cost map} $\M:\C\rightarrow \dR$,
  its \emph{bottleneck cost} is
  \[\bot(\sigma,\M)=\max_{\tau\in [0,1]} \M(\pi(\tau)).\]
\end{definition}
	
We proceed to describe the notion of \emph{robustness}, which is
essential when discussing properties of sampling-based planners.
Given a subset $\Gamma\subset \C$ and two configurations
$x,y\in \Gamma$, denote by $\Pi_{x,y}^\Gamma$ the set of all
continuous paths, whose image is in $\Gamma$, that start in $x$ and
end in $y$, i.e., if $\pi\in \Pi_{x,y}^\Gamma$ then
$\pi:[0,1]\rightarrow \Gamma$ and $\pi(0)=x,\pi(1)=y$.
	
\begin{definition}\label{def:robustly_feasible}
  Let $(\F,s,t)$ be a motion-planning problem. A path
  $\pi \in \Pi_{s,t}^\F$ is \emph{robust} if there exists $\delta >0$
  such that $\B_{\delta}(\pi)\subset \F$.  We also say that $(\F,s,t)$
  is \emph{robustly feasible} if there exists such a robust path.
\end{definition}
	
\begin{definition}\label{def:robust_opt_len}
  The \emph{robust optimum} with respect to $\len$ is defined as
  \[\lenopt=\inf\left\{\len(\pi)\middle|\pi \in \Pi_{s,t}^\F\textup{ is
      robust}\right\}.\]
	\end{definition}
	
	The corresponding definition for the bottleneck cost is slightly
	more involved.
	
	\begin{definition}
		Let $\M$ be a cost map.  A path $\pi \in \Pi_{s,t}^\F$ is
		$\M$-\emph{robust} if it is robust and for every $\eps>0$ there
		exists $\delta >0$ such that for every $x\in \B_{\delta}(\pi)$, $\M(x)\leq (1+\eps)\bot(\pi,\M)$.
		We also say that $\M$ is \emph{well behaved} if there exists at
		least one $\M$-robust path.
	\end{definition}
	
	\begin{definition}\label{def:robust_opt_bot}
		The \emph{robust optimum} with respect to $\bot$ is defined as
		\[\botopt=\inf\left\{\bot(\pi,\M)\middle|\pi \in \Pi_{s,t}^\F\textup{
			is $\M$-robust}\right\}.\]
	\end{definition}
	
	\subsection{Poisson point processes}
	We draw our main analysis techniques from the literature of
    continuum percolation, where point samples are generated with the
    following distribution. Thus we will use this point distribution
    in \prm, which would be formally defined in the following section.
	
	\begin{definition}[\cite{FraMee08}]\label{def:ppp}
      A random set of points $\X\subset \dR^d$ is a \emph{Poisson
        point process} (PPP) of density $\lambda>0$ if it satisfies
      the conditions:  
      \begin{enumerate}
      \item For mutually disjoint domains
        $D_1,\ldots,D_{\ell}\subset \dR^d$, the random variables
        $|D_1\cap \X|,\ldots, |D_{\ell}\cap \X|$ are mutually
        independent.  \item For any bounded domain $D\subset \dR^d$ we
        have that for every $k\geq 0$,
        $$\Pr[|\X\cap D|=k]=e^{-\lambda
          |D|}\tfrac{(\lambda|D|)^k}{k!}.$$
      \end{enumerate}
	\end{definition}

	Informally, property~(1) states that the number of points from
    $\X$ in any two disjoint sets $D_i,D_j\subset \dR^d$ is
    independent.  Another useful property that follows from~(2) is
    that the expected number of points in a certain region is known
    precisely and corresponds to the volume of this region. That is,
    for any $D\subset \dR^d$ it holds that
    $E(|\X\cap D|)=\lambda |D|$.

    The following provides a simple recipe for generating  PPP.

    \begin{claim}[\cite{FraMee08}]
		\label{clm:recipe}
		Let $N$ be a Poisson random variable with mean $\lambda$.  For
        a given $z\in \dZ^d$ draw a sample $N_z\in N$ and define
        $\X_z=\{X_1,\ldots,X_{N_z}\}$ to be $N_z$ points chosen
        independently and uniformly at random from~$z + [0,1]^d$. Then
        $\X=\bigcup_{z\in \dZ^d} \X_z$ is a PPP of density $\lambda$.
	\end{claim}
	
	It will be convenient to think about \prm as a subset of the
    following \emph{random geometric graph} (RGG). We will describe
    various properties of this graph in later sections.

	\begin{definition}{\cite{Pen03}}\label{def:rgg} 
      Let $\X\subset \dR^d$ be a PPP. Given $r >0$, the random
      geometric graph $\G(\X;r)$ is an undirected graph with the
      vertex set $\X$. Given two vertices $x,y\in \X$,
      $(x,y)\in \G(\X;r)$ if \mbox{$\|x-y\|\leq r$}.
	\end{definition}
	
	\section{Analysis of PRM}\label{sec:prm}
	In this section we provide a mathematical description of \prm, which essentially maintains an underlying RGG with PPP samples. 
	We proceed to describe our main contribution (Theorem~\ref{thm:prm}) which is concerned with the conditions for which \prm converges to the (robust) optimum. 
	We then provide an outline of the proof, in preparation for the following sections.
	
	Recall that the configuration space of the robot is represented by
    $\C=[0,1]^d$, and the free space is denoted by $\F\subseteq \C$.
    The motion-planning problem $(\F,s,t)$ will remain fixed
    throughout this section.
	
	Recall that \prm accepts as parameters the number of samples
	$n\in \dN_+$ and a connection radius $r_n$.  Denote by $\X_n$ a
	PPP with mean density $n$.  In relation to the definitions of the
	previous section, the graph data structure obtained by the
	\emph{preprocessing stage} of \prm can be viewed as an RGG. For
	instance, when $\F=\C$, the graph obtained by \prm is precisely
	$\G(\X_n\cap [0,1]^d;r_n)$.  In the more general case, when
	$\F\subset \C$, \prm produces the graph $\G(\X_n\cap \F;r_n)$. As $\F$
	can be non-convex, we emphasize that the latter notation describes the
	maximal subgraph of $\G(\X_n;r_n)$ such that vertices and edges are
	contained in~$\F$.
	
	In the \emph{query stage}, recall that \prm accepts two
    configurations $s,t\in \F$, which are then connected to the
    preprocessed graph. Here we slightly diverge from the standard
    definition of \prm in the literature. In particular, instead of
    using the same radius $r_n$ when connecting $s,t$, we use the
    (possibly larger) radius $r^{st}_n$. The graph obtained after
    query is formally defined below:
	
	\begin{definition}\label{def:prm}
		The \prm graph $\P_n$ is the union between
		$\G(\X_n\cap \F;r_n)$ and the supplementary edges
		\[\bigcup_{y\in \{s,t\}} \left\{(x,y) \middle| x\in \X_n\cap
		\B_{r^{st}_n}(y)\textup{ and } xy\subset \F\right\}. \]
	\end{definition} 
	
	\noindent \textbf{Remark.} We emphasize that the larger radius
	$r^{st}_n$ is only used when $s,t$ are connected to the preprocessed
	RGG. \vspace{5pt}	

	We reach our main contribution. 
	\begin{theorem}\label{thm:prm}
		Suppose that $(\F,s,t)$ is robustly feasible. Then there exists a critical radius 
		$r_n^*=\gamma^*n^{-1/d}$, where $\gamma^*$ is a constant (see supplementary material), such that the following holds:
		\begin{enumerate}[i.]
			\item If $r_n<r^*_n$ and $r^{st}_n=\infty$ then \prm fails (to find a
			solution) \aas\footnote{Let
				$A_1,A_2,\ldots$ be random variables in some probability
				space and let $B$ be an event depending on~$A_n$. We say that
				$B$ occurs \emph{asymptotically almost surely} (\aas, in
				short) if $\limn\Pr[B(A_n)]=1$.}
			\item Suppose that
			$r_n>r^*_n$. There exists $\beta_0>0$ such that for $r^{st}_n=\frac{\beta\log^{1/(d-1)}n}{n^{1/d}}$, where $\beta \geq \beta_0$, and any $\eps>0$ the following holds with probability
			$1-O(n^{-1})$:
			\begin{enumerate}[1.]
				\item $\P_n$ contains a path $\pi_n \in \Pi_{s,t}^\F$
				with \mbox{$\len(\pi_n)\leq(1+\eps)\xi\lenopt$}, where $\xi$ is independent of~$n$;
				\item If $\M$ is well behaved then $\P_n$ contains a path
                  \mbox{$\pi'_n \in \Pi_{s,t}^\F$} with
				$\bot(\pi'_n,\M)\leq (1+\eps) \botopt$.
			\end{enumerate}
		\end{enumerate}
	\end{theorem}

    \noindent\textbf{Remark:} Theorem~\ref{thm:prm} also holds, with a slight
modification, when the PPP is replaced with the standard binomial
point process (BPP), in which the number of sampled points is fixed a priori. The only difference is that the probability of
success should be reduced from $1-O(n^{-1})$ to $1-O(n^{-1/2})$. This
follows from Lemma~1 in~\cite{FriSauSta13}, which states that if a
property holds for PPP then it also holds for BPP, albeit with
slightly smaller probability.
	
	\subsection{Outline of proof}
	For the remainder of this section we briefly describe our technique for proving this theorem, in preparation for the full proof, which is given in Section~\ref{sec:proof}. The critical radius $r^*_n$ defined above, coincides with the \emph{percolation threshold}, which determines the emergence of a connected component of $\G$ that is of infinite size. In particular, if $r_n<r^*_n$ then $\G(\X_n;r_n)$ breaks into tiny connected 
	components of size $O(\log n)$ each. 
	Thus, unless $s,t$ are infinitesimally close, no connected component can have both $s$ and $t$ simultaneously. 
	
	More interestingly, the radius of $r_n>r_n^*$ leads to the emergence
	of a unique infinite component of $\G(\X_n;r_n)$. That is, in
	such a case one of the components of $\G(\X_n;r_n)$ must contain
	an infinite number of vertices (Section~\ref{sec:basics}). Denote this
	component by $C_\infty$. 
	
	In contrast to $\G(\X_n;r_n)$, which is defined for the unbounded space $\dR^d$, our motion-planning problem is bounded to $\C=[0,1]^d$. Thus, the next step is to investigate the properties of $\G(\X_n;r_n)$ when restricted to $[0,1]^d$ (Section~\ref{sec:bounded}). Denote by $C_n$ the largest connected component of $C_\infty \cap [0,1]^d$. This structure plays a key role in our proof (see Section~\ref{sec:proof}): With high probability (to be defined), there exist vertices of $C_n$ that are sufficiently close to $s$ and $t$, respectively, so that a connection between the two vertices can be made through $C_n$. 
	
	Of course, this overlooks the fact that some portions of $C_n$ lie in forbidden regions of $\C$. Thus, we also have to take the structure of $\F$ into consideration. To do so, we rely on~\cite{PenPiz96} to prove that any small subset of $[0,1]^d$ must contain at least one point of $C_n$ (see lemmata~\ref{lem:H_n} and~\ref{lem:H'_n}). This allows us to trace the robust optimum (and collision-free) path with points from~$C_n$. 
	
	The final ingredient, which allows to bound the path length along $\G(\X_n;r_n)\cap [0,1]^d$, is~Theorem~\ref{thm:stretch}. 
	It states that the distance over this graph is proportional to the Euclidean distance between the end points. This also ensures that the trace points from~$C_n$ can be connected with collision-free paths over the graph. 
	
	To conclude, in Section~\ref{sec:elements} we provide background on continuum percolation in unbounded and bounded domains, and prove two key lemmata (Lemma~\ref{lem:H_n} and Lemma~\ref{lem:H'_n}). In Section~\ref{sec:proof} we return to the setting of motion planning and utilize the aforementioned results in the proof of Theorem~\ref{thm:prm}.
	
	\section{Elements of continuum percolation}\label{sec:elements}
	In this section we describe some of the properties of the unbounded
	graph $\G(\X_n;r_n)$ that will be employed in our analysis in the
	following section.
	
	\subsection{The basics}\label{sec:basics} 
	A fundamental question is when $\G$ contains an infinite connected
	component around the origin.
	
	\begin{definition}
		The \emph{percolation probability} $\theta(n,r)$ is the
		probability that the origin $o\in \dR^d$ is contained in a connected
		component of $\G(\X_n\cup \{o\}; r)$ of an infinite number of
		vertices. That is, if $C_o$ denotes the set of vertices connected
		to $o$ in the graph, then $\theta(n,r)=\Pr(|C_o|=
		\infty)$.
	\end{definition}
	
	We say that a graph \emph{percolates} iff $\theta(n,r)>0$. Note
	that the selection of the origin is arbitrary, and the following
	result can be obtained for any  $x\in \dR^d$ alternative to $o$.
	
	\begin{theorem}{\cite[Theorem~12.35]{Gri99}}\label{thm:radius}
		There exists a critical radius $r^*_n=\gamma^* n^{-1/d}$, where
		$\gamma^*$ is a constant, such that $\theta(n,r_n)=0$ when
		$r_n<r^*_n$, and $\theta(n,r_n)>0$ when $r_n>r^*_n$.
	\end{theorem}
	
    The following lemma states that the infinite connected component exists with probability strictly $0$ or $1$. 
	\begin{lemma}\label{lem:psi}
		Let $\psi(n,r)$ be the probability that
		$\G(\X_n; r)$ contains an infinite connected component,
		i.e., without conditioning on any specific additional vertex.  Then
		$\psi(n,r)=0$ when $\theta(n,r)=0$ and $\psi(n,r)=1$
		when $\theta(n,r)>0$.
	\end{lemma}

    \begin{proof}
  Suppose that $\theta(n,r)=0$, and for any $x\in \dR^d$, denote by
  $\theta_x(n,r)$ the percolation probability of $\G(\X_n\cup \{x\};r)$,
  and note that $\theta_x(n,r)=0$. Define
  $Z_r=\left\{r\cdot z|z\in \dZ^d\right\}$, and observe that for any $y\in \dR^d$ there exists $z\in Z_r$ such that $y\in \B_r(z)$. By definition,
  $\G(\X_n\cup\{x\};r)$ percolates, i.e., the infinite component
  touches $\B_r(x)$, iff there exists $y\in C_\infty$ such that
  $\|x-y\|\leq r$. By definition of $Z_r$, if the latter event occurs then there exists $z\in Z_r$ such that $\|z-y\|\leq r$. Thus, by applying
  the union bound, and noting that a sum of countable number of zeros
  is still zero, we establish that
		\begin{align*}
		\psi(n,r)& =\Pr\left[\exists x\in \dR^d, \G(\X_n\cup \{x\};r)\textup{ percolates}\right]\\ &= \Pr\left[\exists z\in Z_r, \G(\X_n\cup \{z\};r)\textup{ percolates}\right] \\ & \leq \sum_{z\in Z_r}\theta_z(n, r)=0.
		\end{align*}
		
		For the other direction we employ Kolmogorov's zero-one law (see,
		e.g.,~\cite[Theorem~1, p.36]{BolRio06}). Informally, it states that
		an event, e.g., existence of an infinite connected component in
		$\G$, that occurs independently of any finite subset of independent
		random variables, e.g., points from $\X_n$, has a ability
		of either $0$ or $1$. Thus, as
		$\psi(n,r)\geq \theta(n,r)$ it immedietly follows that
		$\psi(n,r)=1$ when $\theta(n,r)>0$. 
\end{proof}
	
	The following theorem establishes that the infinite connected component is unique.
	
	\begin{theorem}{\cite[Theorem~2.3]{MeeRoy94}}\label{thm:unique}
		With probability $1$, $\G(\X_n;r)$ contains at most one
		infinite connected component.
	\end{theorem}
	
	\subsection{Bounded domains}\label{sec:bounded}
	We study different properties of $\G(\X_n;r)$ when it is restricted to
	the domain $[0,1]^d$. In case that $\theta(n,r)>0$, we use
	$C_{\infty}$ to refer to the infinite connected component of the
	unbounded graph $\G(\X_n;r)$.  Note that $C_{\infty}$ exists
	(Lemma~\ref{lem:psi}) and is unique (Theorem~\ref{thm:unique})
	with probability $1$.
	
	Denote by $C_n$ the largest connected component of
    $C_\infty\cap [0,1]^d$. By definition, $C_n$ is also a subgraph of
    $\G(\X_n\cap [0,1]^d;r_n)$. The following lemma shows that with
    high probability all the points from $C_\infty$, that are
    sufficiently close to the center of $[0,1]^d$, are members of
    $C_n$.
	
	\begin{lemma}\label{lem:H_n}
		Let $r_n>r^*_n$. Define
		\[H_n=[0,1]^d\setminus \B_{1/ \log n }\left(\dR^d\setminus
          [0,1]^d\right).\]
        Denote by $\EE^1_n$ the event that
        $C_{\infty}\cap H_n\subset C_n$.  Then there exist
        $n_0\in \dN$ and $\alpha>0$ such that for any $n>n_0$ it holds
        that
		\[\Pr[\EE^1_n]\geq 1-\exp\left(-\alpha n^{1/d}\log^{-1}n\right).\] 
	\end{lemma}
	
	\begin{proof}
      This statement is an adaptation of Lemma~8 and Theorem~2
      of~\cite{PenPiz96}. We mention that~\cite{PenPiz96} uses
        a slightly different, but nevertheless equivalent model. While
        we consider an RGG that is bounded to $[0,1]^d$ and PPP of
        density $n$, they consider the domain $[0,n]^d$ with density
        $1$. It is only a matter of rescaling and variable
        substitution to import their results to our domain.  Let
      $C'_1,\ldots,C'_k$ denote the connected components of
      $\C_\infty\cap [0,1]^d$, and set $C_n$ to be the component
      $C'_i$ with the largest number of vertices. Without loss of
      generality, $C_n=C'_1$. See illustration in
      Figure~\ref{fig:h_n}.
		
      Observe that if $C'_i\subset C_\infty$ then it must have at
      least one vertex $x'_i\in C'_i$ that lies closely to the
      boundary of $[0,1]^d$, i.e., $\|x-\partial([0,1]^d)\|\leq r_n$,
      as otherwise $C'_i$ will not be able to connect to the rest of
      $C_\infty$. Thus, it must be the case that
      $\textup{diam}(C'_i)\leq \log^{-1} n-r_n$ in order to be able to
      reach $H_n$, where $\textup{diam}(D)=\sup_{x,x'\in D}\|x-x'\|$
      defines the diameter of a given $D\subset \dR^d$. We shall show
      that this does not hold.
		
      Theorem~2 in~\cite{PenPiz96} states that for any $\phi_n$ large
      enough there exist $\alpha',n_0$ such that for any $n>n_0$ it
      holds with probability at least
      $1-\exp\left(-\alpha' n^{1/d}\phi_n\right)$ that $\textup{diam}(C'_i)<\phi_n$ for any
      $1<i\leq k$. By setting
      $\alpha=\alpha'/2,\phi_n=1/(2\log n)$ the conclusion immediately
      follows.
	\end{proof}
	
	\begin{figure}
		\centering
		\includegraphics[width=0.5\columnwidth]{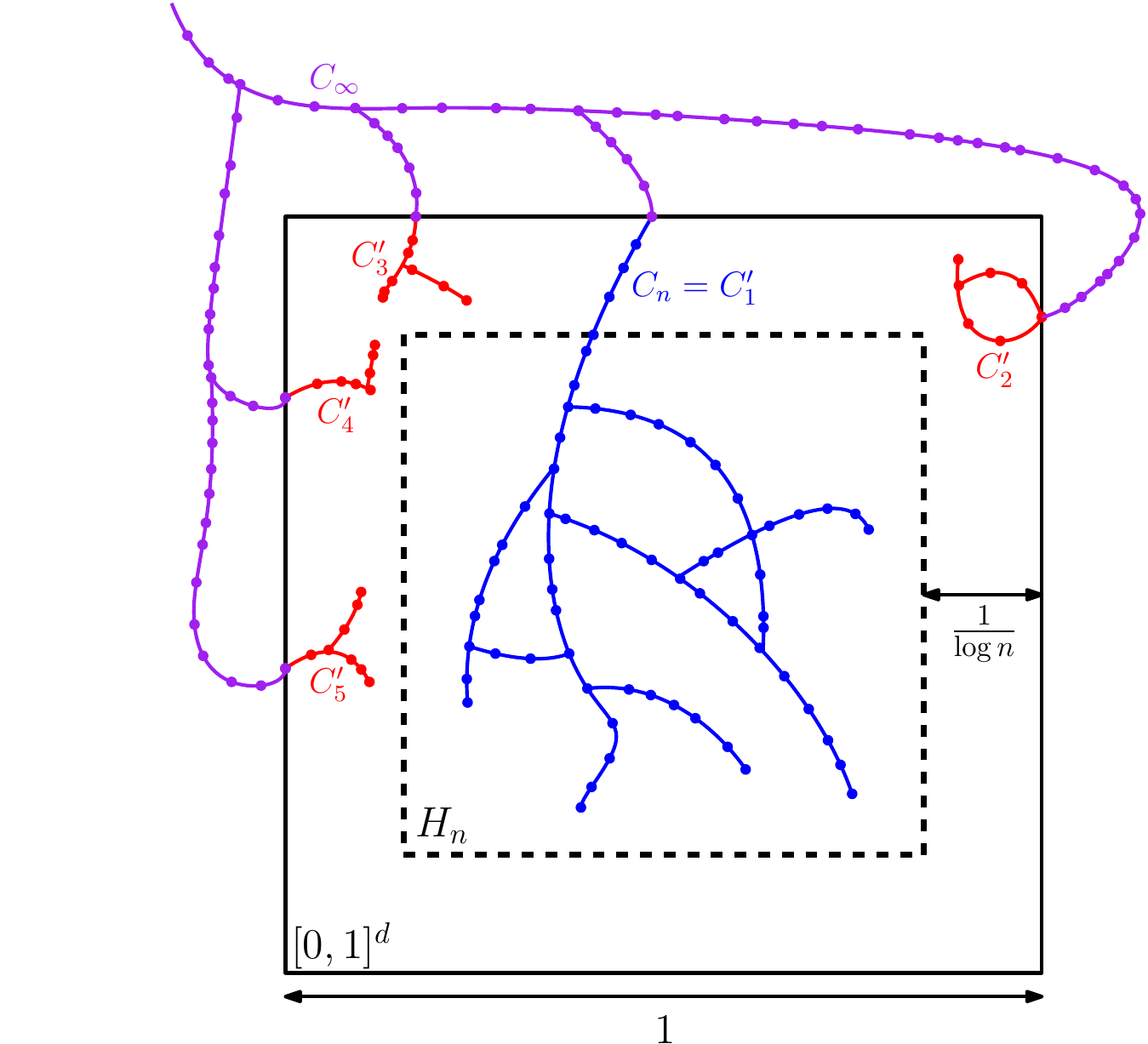}
		\caption{Illustration for Lemma~\ref{lem:H_n}. The outer cube
			represents $[0,1]^d$, whereas the inner cube depicted with dashed
			boundary is $H_n$. Observe that $H_n$ has  side length of
			$1-2/\log n$. The purple, blue and red graphs combined describe
			$C_\infty$, whereas $C_n$ is depicted in blue, and
			$C'_2,C'_3,C'_4,C'_5$ are in red.\label{fig:h_n}}
	\end{figure}
	
	Observe that for any fixed point $x$ internal to $[0,1]^d$ there
    exists $n_0\in (0,\infty)$ such that for any $n>n_0$ it holds that
    $x\in H_n$.  The following lemma bounds the probability of having
    at least one point from $C_n$ in a (small) subregion of
    $H_n$. Note that the value $\beta$ corresponds to the value used
    in Theorem~\ref{thm:prm}.
	\begin{lemma}\label{lem:H'_n}
		Let $r_n>r^*_n$. Define $H'_n\subset H_n$ to be a hypercube of side
		length $$h'_n=\tfrac{\beta\log^{1/(d-1)}n}{n^{1/d}}.$$
		Denote by $\EE^2_n$ the event that $H'_n\cap C_n\neq \emptyset$.
		Then there exists $n_0\in \dN$ and $\beta_0>0$ such that for
		any $n>n_0,\beta>\beta_0$ it holds that
		$\Pr[\EE^2_n | \EE^1_n]\geq 1-n^{-1}$.
	\end{lemma}
	\begin{proof}
      Define $G_n=\G(\X_n,r_n)\cap H'_n$, $n'=E[|\X_n\cap H'_n|]$ and
      observe that $n'= n\cdot |H'_n|=\beta^d \log^{d/(d-1)}n$.
		
      We treat $G_n$ as a subset of $\dR^d$ in order to apply a
      rescaling argument.  Observe that (scalar) multiplication of
      every point of $H'_n$ with $1/h'_n$ yields a translation of
      $[0,1]^d$. We will use the superscript $1/h'_n$ to describe this
      rescaling to a given object. For instance, applying the same
      transformation on $G_n$ yields the graph $G_n^{1/h'_n}$, which
      has the same topology as $G_n$.  Denote by $x_{\ell}\in H'_n$
      the (lexicographically) smallest point of $H'_n$, i.e.,
      $x_{\ell}=(1/\log n,\ldots,1/\log n)$.  Notice that
      \begin{align*}G_n^{1/h'_n}-x_{\ell}^{1/h'_n} &=\G(\X^{1/h'_n}_n\cap
      [0,1]^d,r_n/h'_n) \\ &=\G(\X_{n'}\cap [0,1]^d,r_{n'}),\end{align*}
      where the minus sign in the left-hand side represents a
      translation by a vector. This implies that $G_n$, which is
      defined over $H'_n$ behaves as $\G(\X_{n'}\cap
      [0,1]^d;r_{n'})$.
      This allows to leverage Theorem~1 from~\cite{PenPiz96}, which
      bounds the number of vertices from the unbounded component. In
      particular, there exists $\beta_0>0$ such that
		\begin{align*}
		\Pr[C_\infty\cap H'_n=\Theta(n')] & \geq                                 1-\exp\left(-\beta_0^{-(d-1)}{n'}^{\frac{d-1}{d}}\right)
		\\ & = 1-\exp\left(-\beta_0^{-(d-1)}\beta^{d-1}\log n\right)
		\\ & \geq 1-\exp(-\log n)= 1-n^{-1}.
		\end{align*}
		While this is an overkill for our purpose, it does the job in proving that 
		$\Pr[C_\infty\cap H'_n\neq \emptyset]\geq 1-n^{-1}$. As we assume that $\EE^1_n$ holds (Lemma~\ref{lem:H_n}), it follows
		that $C_n\cap H'_n\neq \emptyset$ holds with probability at
		least $1-n^{-1}$. 
	\end{proof}
	
	The following statement allows to bound the graph distance between two connected vertices. We endow every edge of the
	graph with a length attribute that represents the Euclidean distance
	between the edges' endpoints. For every two vertices $x,x'$ of $\G$,
	$\textup{dist}(\G,x,x')$ denotes the length of the shortest (weighted) path on $\G$ between the two vertices.
	
	\begin{theorem}{\cite[Theorem~3]{FriSauSta13}}\label{thm:stretch}
		Let $r_nr^*_n$. There exists a constant
		$\xi\geq 1$, independent of $n$, such that $\Pr[\EE^3_n]=1-O(n^{-1})$,
		where the event~$\EE^3_n$ is defined as follows:
		For any two vertices $x, x'$ in the same connected
		component of $\G(\X_n\cap [0,1]^d;r_n)$, with
		$ \|x - x'\| = \omega(r_n)$, it holds that
		$\dist(\G_n,x,x') \leq \xi\|x-x'\|$.
	\end{theorem}
	
	\section{Proof of Theorem~\ref{thm:prm}}\label{sec:proof}
    Proofs for all the three settings of the main theorem are given individually in the subsections below.

    \subsection{Case i}
Here we provide proof for the first (and easy) part of the theorem,
which states that \prm fails when $r_n<\gamma^* n^{-1/d}$, even with
$r^{st}_n=\infty$. We mention that our proof is a simpler and shorter
version of a similar proof that was given in~\cite{KF11} for a
slightly different setting.

We show below that in the subcritical regime, i.e., $r_n<r^*_n$,
the graph breaks into many small connected components. In particular,
the probability of having the largest connected component  of
size $m$ decays exponentially in $m$. Denote by $L(\G)$ the size of
the largest connected component in $\G$.
	
\begin{proposition}\label{prop:log}
  Let $r_n<r^*_n$. Then $L(\G_n)=O(\log n)$ \aas
\end{proposition}
\begin{proof}
  First, we start by stating that there exist a constant $\zeta>0$ and
  an integer $m_0$ such that for all $m\geq m_0$ it holds that
  $\Pr[L(\G_n)\geq m]\leq n e^{-m\zeta}$.  This is a simplified
  version of Proposition~11.2 in~\cite{Pen03}, which is given for $n$
  uniformly sampled points, to the case of PPP, which we have here. In
  particular, the original proof relies on a relation between these
  two distributions, which transforms uniform sampling into a PPP and
  induces an additional factor that is not necessary in our
  setting. In particular, the factor ``$\textup{exp}(-\mu n)$'' which
  appears in Equation~11.1 in~\cite{Pen03} should be eliminated.
		
  To conclude the proof, we set $m=\alpha \cdot \log n$, where
  $\alpha >1/\zeta$, similarly to the proof of Theorem~11.1,
  Equation~11.4 in~\cite{Pen03}. Observe that
  $ne^{-m\zeta}\rightarrow 0$ as $n$ tends to $\infty$. 
\end{proof}

Now, consider the configuration space $C=[0,1]^2$ depicted in
Figure~\ref{fig:incomplete} (similar examples can be devised for any
$d\geq 2$). The white and blue regions represent the free
space. Observe that any path connecting $s$ to $t$ must go through the
blue region, whose width is greater than $1/2$. Also note that any
point contained in the blue region cannot be connected by an edge from
$s$ or $t$, which deems the large value of $r^{st}_n$ as irrelevant.
\begin{figure}[h]
  \centering
  \includegraphics[width=0.4\columnwidth]{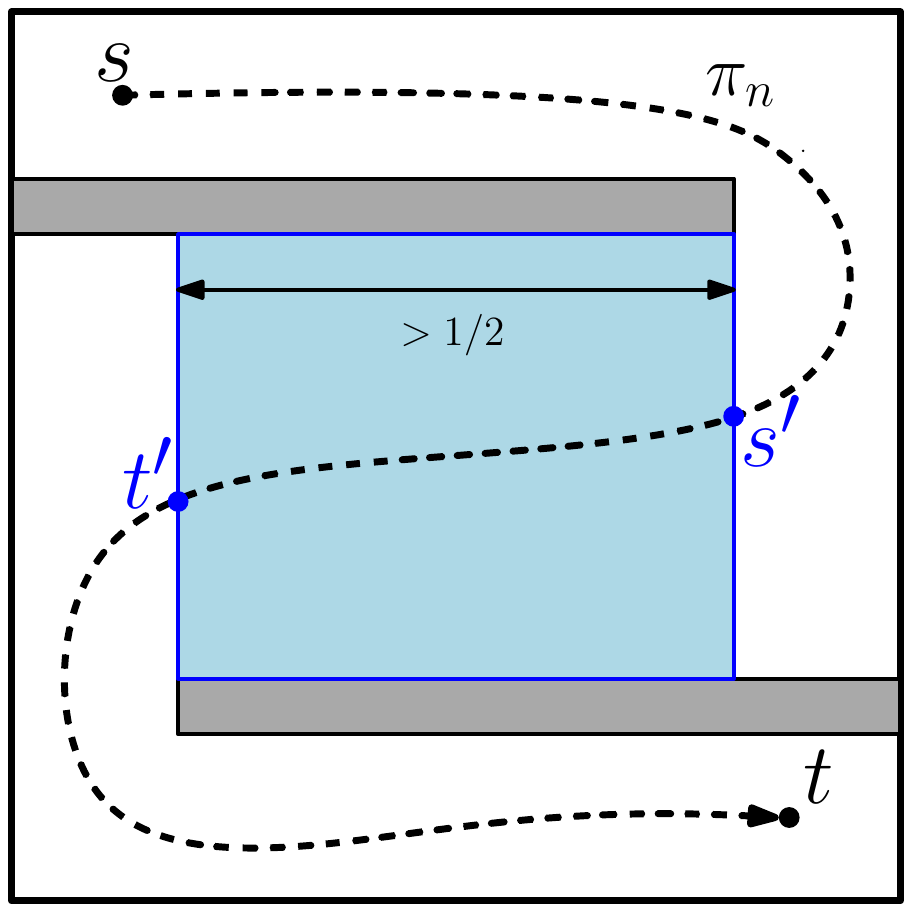}
  \caption{Scenario for the proof of
    Theorem~\ref{thm:prm}.i.\label{fig:incomplete}}
\end{figure}
	
Suppose that $\G_n=\G(\X_n,r_n)$ contains a path $\pi_n$ (black dashed
curve) connecting~$s$ and~$t$, and denote by $\pi'_n\subset \pi_n$ the
subpath that is contained in the blue region. Also, denote by $s'$ and
$t'$ the first and last points along $\pi'_n$, respectively. Obviously
$1/2\leq \|s'-t'\|\leq \len(\pi_n)$. Then, the number of edges of
$\G_n$, which induce $\pi'_n$, is at least
$\|s'-t'\| / r_n = \Omega(n^{1/d})$.  However, this is in
contradiction with the fact that every connected component of $\G_n$
is of logarithmic size in $n$ (Proposition~\ref{prop:log}).

    \subsection{Case ii.1}
	We provide a full proof for the positive setting with length cost.
    The proof for the bottleneck case, which appears later on, is very similar to the length
    case.

	Suppose that
    $r_n>r^*_n,r^{st}_n=\frac{\beta\log^{1/(d-1)}n}{n^{1/d}}$, $\beta>\beta_0$.
    For simplicity, we set $r_n=\gamma n^{-1/d}$, where
    $\gamma>\gamma^*$.  By Lemma~\ref{lem:psi} and
    Theorem~\ref{thm:unique}, $\G(\X_n; r_n)$ contains a unique
    infinite connected component $C_\infty$. Recall that $C_n$ denotes
    the largest connected component of $C_\infty\cap [0,1]^d$. Also
    note that $r^{st}_n= h'_n$, where $h'_n$ is defined in
    Lemma~\ref{lem:H'_n}.
	
	Recall that $\lenopt$ denotes the robust optimum, with respect to
	path length (Definition~\ref{def:robust_opt_len}). Fix
	$\eps>0$. By definition, there exists a robust path
	$\pi_{\eps}\in \Pi_{s,t}^\F$ and $\delta>0$ such that
	$\len(\pi_{\eps})\leq (1+\eps)\lenopt$ and
	$\B_{\delta}(\pi_\eps)\subset \F$. See illustration in Figure~\ref{fig:complete}.
	
	\begin{figure}
		\centering
		\includegraphics[width=0.5\columnwidth]{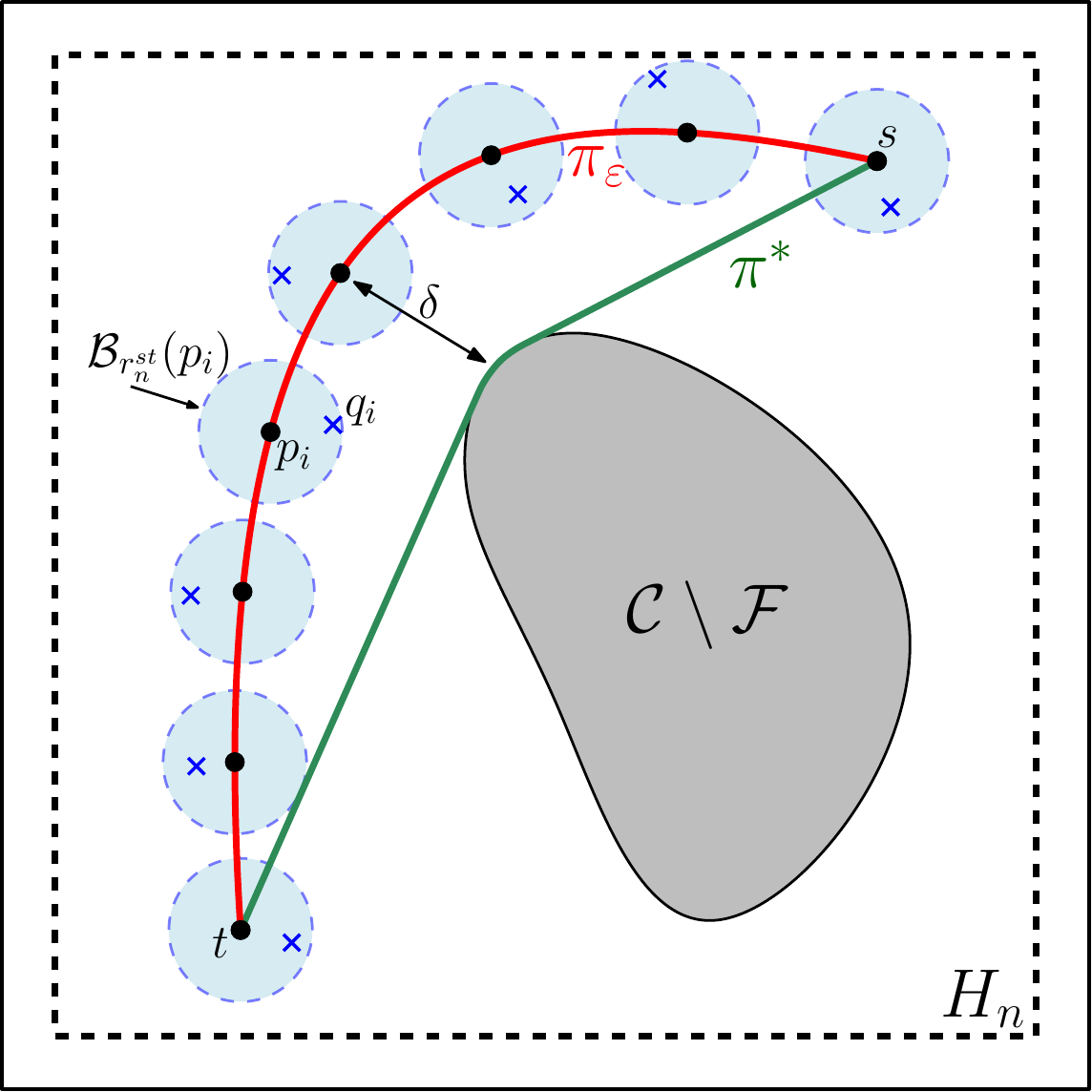}
		\caption{Illustration for the proof of
			Theorem~\ref{thm:prm}.ii. The outer cube represents the configuration space, whereas the dashed cube is $H_n$. The gray area represents the forbidden regions. The robust feasible path $\pi^*$ and $\pi_\eps$ are depicted in green and red, respectively. Observe that every point along $\pi_\eps$ is at least $\delta$ away from $\C\setminus \F$.  $p_1=s,p_2,\ldots,p_{k-1},p_k=t$ are depicted as black bullets, where $k=8$. $\B_{r^{st}_n}(p_i)$ are depicted as blue circles, while the blue cross in each such circle represents $q_i$. \label{fig:complete}}
	\end{figure}
	
	We now define a sequence of $k$ points $p_1,\ldots,p_k$ along
	$\pi_\eps$ that are separated by exactly $\delta/2\xi$ units,
	where $\xi$ is as defined in Theorem~\ref{thm:stretch}. In
	particular, define $k=\left\lceil\len(\pi_{\eps})\cdot 2\xi / \delta\right\rceil$,
	set $p_1=s,p_k=t$, and assign $p_i$ along $\pi_\eps$, such that
	$\len\left(\pi_\eps^{i-1,i}\right)=\delta/2\xi$, where
	$\pi_\eps^{i-1,i}$ represents the subpath of $\pi_\eps$ starting
	at $p_{i-1}$ and ending at $p_i$. Notice that $k$ is finite.
	
	\begin{claim}\label{clm:close_points}
		Denote by $\EE^4_n$ the event that for all $i\in [k]$ there
		exists $q_i\in C_n$ such that $q_i\in \B_{r^{st}_n}(p_i)$ and
		$q_i\in C_n$. Then
		$\Pr[\EE^4_n|\EE^1_n]\geq 1-k\Pr[\overline{\EE^2_n}|\EE^1_n]$ (see
		definition of $\EE^1_n,\EE^2_n$ in Lemma~\ref{lem:H_n} and Lemma~\ref{lem:H'_n}, respectively).
	\end{claim}
	\begin{proof}
		Define
		$H'_n(x)\subset \dR^d$ to represent a $d$-dimensional (axis-aligned)
		hypercube of side length $h'_n$ that is centered in $x\in
		\dR^d$. Formally, $H'_n(x)=x+h'_n
		\cdot\left[-\tfrac{1}{2},\tfrac{1}{2}\right]^d$.
		Observe that $H'_n(p_i)\subset H_n$ for $n$ large enough. Also note that $H'_n(p_i)\subset\B_{r^{st}_n}(p_i)$. Thus, the result
		follows from the union bound. 
	\end{proof}
	
	Suppose that $\EE^1_n,\EE^4_n$ are satisfied.  Let
	$q_1,\ldots,q_k\in C_n$ be the points obtained from
	Claim~\ref{clm:close_points}. These points reside in a single
	connected component of $\G_n$. Define the path
	\[\pi_n:=s\rightarrow q_1 \leadsto q_2\leadsto\ldots\leadsto q_k\rightarrow
	t,\]
	where $s\rightarrow q_1$ represents a straight-line path from $s$
	to $q_1$, and $q_i\leadsto q_{i+1}$ represents the shortest path
	from $q_i$ to $q_{i+1}$ in $\G_n$. The next claim states that if
	in addition $\EE^3_n$ is satisfied (Theorem~\ref{thm:stretch}),
	then $\pi_n$ is also collision free.
	
	\begin{claim}\label{clm:free}
		Suppose that $\EE^1_n,\EE^3_n,\EE^4_n$ are satisfied. Then
		$\pi_n\in \Pi_{s,t}^\F$ is a path in $\P_n$ (with probability~$1$).
	\end{claim}
	\begin{proof}
      First, observe that the straight-line paths
      $s\rightarrow q_1,q_k\rightarrow t$ are contained in $\P_n$ due
      to the definition of \prm and the fact that $r_n^{st}<\delta$.
      Let us consider a specific subpath $q_i\leadsto q_{i+1}$. Recall
      from Claim~\ref{clm:close_points} and by definition of
      $p_i,p_{i+1}$ that
      $\dist(\G_n,q_i,q_{i+1}) \leq 2r_n^{st}+\xi \|p_{i+1}-p_i\| =
      o(1)+ \delta/2$, which is bounded by $\delta$. Thus, for any point $q'_i$ along $q_i\leadsto q_{i+1}$ it
      holds that $\|q'_i-p_i\|<\delta, \|q'_i-p_{i+1}\|<\delta$.
      Thus, $\Im(\pi_n)\subset \B_{\delta}(\pi_\eps)\subset \F$. 
	\end{proof}
	
	The next claim states that the length of $\pi_n$ is a constant
	factor from the optimum.
	\begin{claim}\label{clm:short}
		Suppose that $\EE^1_n,\EE^3_n,\EE^4_n$ are satisfied. We have
		that $\len(\pi_n)\leq (1+\eps)\xi \lenopt$ (with probability $1$).
	\end{claim}
	\begin{proof}
      First, observe that by the triangle inequality, it holds that 
      \begin{align*}
        \|q_{i-1} & -q_i\| \\ & \leq \|q_{i-1}-p_{i-1}\|+\|p_i-p_{i-1}\|+	\|q_{i}-p_{i}\| \\ & \leq \|q_{i-1}-p_{i-1}\|+\len(\pi_\eps^{i-1,i})+	\|q_{i}-p_{i}\| \\ & \leq 2r_n + \len(\pi_\eps^{i-1,i}) =o(1) + \len(\pi_\eps^{i-1,i}).
      \end{align*}
By Theorem~\ref{thm:stretch} and the triangle inequality, it follows that
		\begin{align*} \len( \pi_n) & = \|s-q_1\|+ \dist(\G_n,q_1,q_k) +
		\|t-q_k\| \\ & \leq 2r^{st}_n+ \sum_{i=2}^k\dist(\G_n,q_{i-1},q_i)
		\\ &  \leq o(1)+
		\sum_{i=2}^k \xi \|q_{i-1}-q_i\| \\ &  \leq  o(1) + \xi \sum_{i=2}^k
		\len(\pi_\eps^{i-1,i}) \\ & = o(1) + \xi \len(\pi_\eps)\leq
		o(1) + (1+\eps)\xi \lenopt.
		\end{align*}
		We note that in order to eliminate the $o(1)$ summand one can
        fix $\eps$ but work with $\pi_{\eps'}$ instead of
        $\pi_{\eps}$, where $\eps'<\eps$. For simplicity, we chose to
        leave the proof as is. 
	\end{proof}
	
	To conclude, Claim~\ref{clm:free} and Claim~\ref{clm:short} show
	the existence of a (collision free) path $\pi_n$ in $\P_n$, whose
	length is at most $(1+\eps)\xi\lenopt$. It remains to bound the
	probability that $\EE^1_n,\EE^3_n,\EE^4_n$ hold simultaneously:
	\begin{align*}
	\Pr[\EE^1_n & \wedge  \EE^3_n\wedge
	\EE^4_n] =\Pr\left[\EE^3_n\wedge
	\EE^4_n | \EE^1_n\right]\cdot \Pr[\EE^1_n] \\
	&= \left(1-\Pr\left[\overline{\EE^3_n\wedge\EE^4_n} | \EE^1_n\right]\right)\cdot\Pr[\EE^1_n]\\
	&= \left(1-\Pr\left[\overline{\EE^3_n}\vee\overline{\EE^4_n} | \EE^1_n\right]\right)\cdot\Pr[\EE^1_n] \\
	&\geq \left(1-\Pr\left[\overline{\EE^3_n} | \EE^1_n\right] - \Pr\left[\overline{\EE^4_n} | \EE^1_n\right]\right)\cdot\Pr[\EE^1_n]
	\end{align*}
	By Claim~\ref{clm:close_points} and Lemma~\ref{lem:H'_n}, $\Pr[\overline{\EE^4_n}| \EE^1_n] \leq k\Pr\left[\overline{\EE^2_n}| \EE^1_n\right]\leq kn^{-1}$.
	Also note that $\Pr[\EE^3_n| \EE^1_n] \geq \Pr[\EE^3_n]$, since $\Pr[\EE^3_n| \overline{\EE^1_n}] = 0$.
	Therefore, 
	\begin{align*}
	\Pr&[ \EE^1_n  \wedge \EE^3_n\wedge
	\EE^4_n]
	\\ &\geq \left(1-\Pr\left[\overline{\EE^3_n} | \EE^1_n\right] - \Pr\left[\overline{\EE^4_n} | \EE^1_n\right]\right)\cdot\Pr[\EE^1_n] \\
	&\geq \left(1 - O(n^{-1}) - kn^{-1}\right)\left(1-\exp\left(-\alpha n^{1/d}\log^{-1}n\right)\right)\\
	& = 1-O(n^{-1}).
	\end{align*}

\subsection{Case ii.2}
We prove the asymptotic optimality of \prm for bottleneck cost, with appropriate values of $r_n,r^{st}_n$.
	The proof follows very similar lines to that of the length
	cost. Fix $\eps>0$. For simplicity we assume that $\eps<1$ (the
	proof can be adapted to larger values of $\eps$).  By definition,
	there exists an $\M$-robust path $\pi'\in \Pi_{s,t}^\F$ and
	$\delta,\delta'>0$ such that
	\begin{enumerate}[(a)]
		\item $\B_\delta(\pi')\subset\F$, 
		\item $\bot(\pi',\M)\leq (1+\eps/3)\botopt$, 
		\item and $ \forall x\in\B_{\delta'}(\pi'):\M(x)\leq
		(1+\eps/3)\bot(\pi',\M)$.
	\end{enumerate}
	
	Now, define $\delta^*=\min\{\delta,\delta'\}$. By substituting
	$\delta$ with $\delta^*$, and $\pi_\eps$ with $\pi'$ in the proof
	of Theorem~\ref{thm:prm}.ii, it follows that if
	$\EE^1_n,\EE^3_n,\EE^4_n$ are satisfied, $\P_n$ contains a path
	$\pi'_n\in \Pi_{s,t}^\F$ such that
	$\Im(\pi'_n)\subset \B_{\delta'}(\pi')$. This implies that
    \begin{align*}
	\bot(\pi'_n,\M)&\leq (1+\eps/3)\bot(\pi',\M) \\ & \leq
	(1+\eps/3)^2\botopt\leq (1+\eps)\botopt,
    \end{align*}
	which concludes this proof. 

\section{Analysis of other planners}\label{sup:other_planners}
	In this section we describe the implications of Theorem~\ref{thm:prm}
	to additional planners, which are closely related to \prm. In all the
	results below, the values $r^*_n,\gamma^*,\beta, \xi$ are identical to
	those in Theorem~\ref{thm:prm}. Furthermore, for simplicity of
	presentation, we assume that $(\F,s,t)$ is robustly feasible, and $\M$
	is well behaved, where relevant.
	
	
	We first consider \fmt and \btt, and then proceed to \rrg, whose
	analysis is more involved.  In all the following algorithms we
	assume that the sampling method is a PPP.
	
	\subsection{The \fmt and the \btt planners}
	The following two corollaries are direct consequences of
	Theorem~\ref{thm:prm} as the two algorithms \fmt and \btt
	traverse an implicitly-represented \prm graph while minimizing
	the cost functions $\len$ and $\bot$, respectively. 
	\begin{corollary}\label{cor:fmt}
		If $r_n<r^*_n$ and $r^{st}_n=\infty$ then \fmt fails \aas If
		$r_n>r^*_n$ and $r^{st}_n=\frac{\beta\log^{1/(d-1)}n}{n^{1/d}}$, then for any $\eps>0$ \fmt returns a path
		$\pi_n \in \Pi_{s,t}^\F$ with $\len(\pi_n)\leq(1+\eps)\xi\lenopt$ with probability $1-O(n^{-1})$.
	\end{corollary}
	
	We remark that \fmt can ignore certain edges of the underlying
	\prm graph, which may result in a path of lower quality than that
	obtained by \prm. However, this situation only occurs for
	vertices which lie within distance of $r_n$ from $\C\setminus \F$
	(see~\cite[Remark~3.3]{JSCP15}). As our analysis of \prm in the
	previous section utilizes only vertices and edges that are far from
	obstacles our proofs extend to \fmt straightforwardly.
	\begin{corollary}\label{cor:btt}
		If $r_n<r^*_n$ and $r^{st}_n=\infty$ then \btt fails \aas If
		$r_n>r^*_n$ and $r^{st}_n=\frac{\beta\log^{1/(d-1)}n}{n^{1/d}}$, then for any $\eps>0$ \btt returns a path
		$\pi'_n \in \Pi_{s,t}^\F$ with
		$\bot(\pi'_n,\M)\leq (1+\eps) \botopt$ with probability $1-O(n^{-1})$.
	\end{corollary}
	
	\subsection{The \rrg planner}
	We consider the incremental planner \rrg, which can be viewed as
	a cross between \rrt and \prm. Due to this relation we can extend
	the analysis of \prm to \rrg. (The term ``incremental'' refers to planners that generate samples one after the other, and connect the
		current samples to previous ones.)
	
	We introduce an incremental version of a PPP to extend our theory
	to \rrg.
	
	\begin{claim}
		Let $N$ be a Poisson random variable with mean $1$.  At each
		iteration $i$ draw a sample $N_i\in N$ and define
		$\X_i=\{X_1,\ldots,X_{N_i}\}$ to be $N_i$ points chosen
		independently and uniformly at random from~$[0,1]^d$.  The set
		$\X=\bigcup_{1\leq i\leq n} \X_i$, obtained after $n$ iterations, is
		a PPP of density $\lambda = n$.
		\label{clm:inc_ppp}
	\end{claim}
	
	Note that the sum of $n$ i.i.d.\ Poisson random variables with
	means $\lambda_1, \lambda_2, \ldots, \lambda_n$ is a Poisson
	random variable with mean $\lambda = \sum_{i=1}^{n}{\lambda_i}$.
	Claim~\ref{clm:inc_ppp} follows directly from this property.
	
	We now describe an adaptation of \rrg for an incremental PPP.  Given a
	start configuration $s\in \F$ and 
	a goal region $\X_{\text{goal}}\subseteq \F$, \rrg initializes a
	roadmap with a single node $s$.  Let~$\eta$ denote the constant used
	by \rrg for local steering (see~\cite{KF11}). Similarly to \prm, \rrg employs the connection radii $r_n^{\rrg},{r_n^s}^{\rrg}$.
	
	Let $N$ be a Poisson random variable with mean $1$ and let $n_{i-1}$
	denote the number of nodes of the constructed roadmap after $i-1$
	iterations.  At the $i$th iteration we draw a sample $N_i\in N$ and
	define $\X_i=\{X_1,\ldots,X_{N_i}\}$ to be~$N_i$ points chosen
	independently and uniformly at random from~$[0,1]^d$.  \rrg will
	iteratively process the samples $X_1,\ldots,X_{N_i}$, as follows: For
	a sample $X_k\in \X_i$ \rrg will first locate the nearest node
	$x_{\text{near}}$ to $X_k$.  Let $x_{\text{new}}$ be a node at
	distance at most~$\eta$ from $x_{\text{near}}$ at the direction of
	$X_k$.  \rrg will attempt to connect $x_{\text{new}}$ to
	$x_{\text{near}}$.  If the connection attempt is successful, the new
	node $x_{\text{new}}$ will be added to the roadmap. Then, \rrg will
	attempt to connect $x_{\text{new}}$ to all the existing nodes within a
	ball of radius $\min\{r_n^\rrg, \eta\}$. It will additionally attempt to connect $x_{\text{new}}$ to $s$ if $\|x_{\text{new}}-s\|\leq \min\{{r_n^s}^{\rrg},\mu\}$. 
	Note that for the $i$th iteration of \rrg we set $n=n_{i-1}$.
	
	The connection radii $r^{\rrg}_{n}, {r^{s}_n}^{\rrg}$ that will be
	used during the $i$th
	iteration 
	are fixed, as both are functions of the number of nodes in the roadmap
	at the beginning of the 
	iteration.  Both radii will be set to decrease with the number of
	roadmap nodes, as in \prm.  However, since we fix the radius at each
	iteration of \rrg we use a slightly larger radius than the one
	obtained had we considered the current roadmap size.  This will
	clearly keep all relevant connections and perhaps even add more.
	
	\begin{theorem}\label{thm:rrg}
		Suppose that $(\F,s,\X_{\text{goal}} )$ is robustly
		feasible (we extend Definition~\ref{def:robustly_feasible}
			to describe a robustly feasible path $(\F ,s, \X_{\text{goal}})$
			for \rrg).  Let $\R_n$ denote the roadmap constructed by \rrg
		after $n$ iterations with
		\[r^{\rrg}_n=(1+\mu)r_n,\quad {r^{s}_n}^{\rrg}=(1+\mu)r^{st}_n,\]
		where $r_n,r^{st}_n$ are as in Theorem~\ref{thm:prm}.ii, and $\mu$
		is any positive constant. Then for any~$\eps>0$ the following holds
		with probability $1-O(n^{-1})$:
		\begin{enumerate}[1.]
			\item $\R_n$ contains a path $\pi_n \in \Pi_{s,\X_{\text{goal}}}^\F$
			with \mbox{$\len(\pi_n)\leq(1+\eps)\xi\lenopt$};
			\item If $\M$ is well behaved then, $\R_n$ contains a path
		\mbox{$\pi'_n \in \Pi_{s,\X_{\text{goal}}}^\F$} with
			$\bot(\pi'_n,\M)\leq (1+\eps) \botopt$.
		\end{enumerate}
	\end{theorem}
	
	\subsubsection{Proof of Theorem~\ref{thm:rrg}}
	We provide proof only for $\len$, as the case for $\bot$ is almost
	identical. In what follows we consider a robust path
	$\pi_\eps\in\Pi_{s,\X_{\text{goal}}}^\F$ for which
	$\len(\pi_\eps)\leq (1+\eps)\lenopt$.  Let $\delta>0$ denote the
	clearance of $\pi_\eps$, and $L$ denote its length.  Finally, set
	$\kappa = \min{(\delta, \eta)}/5$.
	
	Our proof uses Theorem~\ref{thm:prm} as a central ingredient, due to
	the following observation, which is formalized below: For $n$ large
	enough, with probability $1-O(e^{-n})$, $\R_n$ contains a \prm roadmap
	$\P_{n'}$ in the vicinity of $\pi_\eps$, where $n'$ is slightly
	smaller than $n$. This in turn follows from the unique structure of \rrg, which can be viewed as a combination of \prm and \rrt, as \rrg supplements the \rrt tree with additional edges. 
	
	In particular, the following corollary can be deduced from the
	probabilistic completeness of \rrt with samples from a PPP (see Appendix~\ref{sec:rrt_pc}), which does not
	rely on the parameters $r^{\rrg}_n, {r^{s}_n}^{\rrg}$.
	\begin{claim}\label{claim:rrt}
		It is possible to tile $\pi_\eps$ 
          with $m = L/\kappa$ balls of
		radius $\kappa$, such that after $n$ iterations of \rrg, with
		probability at least $1-a\cdot e^{-bn}$ for some constants
		$a,b\in \dR_{>0}$, every ball will contain at least one \rrg vertex.
	\end{claim}
	
	The following lemma formalizes the connection between \rrg and \prm. 
	\begin{lemma}\label{lem:rrg_contains_prm}
		Fix $n$ and define $n'=(1+\mu)^{-d}n$. Let $\R_n$ be an \rrg graph
		constructed after $n$ iterations with connection radii
		$r^{\rrg}_{n}, {r^{s}_n}^{\rrg}$.  With probability at least
		$1-a\cdot e^{-b'n}$, for some constants $a,b'\in\dR_{>0}$, the
		graph $\R_n\cap \B_{\kappa/2}(\pi)$ contains a \prm graph
		$\P_{n'}\cap \B_{\kappa/2}(\pi)$, constructed with the radius
		$r_{n'}$.
	\end{lemma}
	\begin{proof}	
      Consider the $n-n'=n\left(1-(1+\mu)^{-d}\right)$ first
      iterations of \rrg.  From Claim~\ref{claim:rrt} we obtain
      that given a robust path $\pi_\eps$, it can be tiled using a
      constant number of balls of radius $\kappa$ such that with
      probability of at least $1-ae^{-b(n-n')}=1-ae^{-b'n}$ every ball
      will contain an \rrg vertex after $n$ iterations of \rrg, for
      some constants $a,b\in \dR_{>0}$, and
      $b'=b\left(1-(1+\mu)^{-d}\right)$. Now, let us consider a new
      sample $x \in \B_{\kappa/2}(\pi)$ that is added in iteration
      $n''>n-n'$. By Claim~\ref{claim:rrt}, $\R_{n''}$ includes
      (with certain probability) a node $v$ that can be connected to
      $x$; that is, both $\|x-v\|\leq \eta$ and the straight-line path
      from $x$ to $v$ is collision free. Thus, $v$ is a added as a
      node to \rrg and then connected in a \prm-fashion to all its
      neighbors within a radius of $r_{n''}^{\rrg}$.
		
      Note that the smallest value of the latter radius is when
      $n''=n$. Observe that
      \begin{align*}
      r_{n}^{\rrg}=(1+\mu)r_n & = \gamma \left((1+\mu)^{-d}n\right)^{-1/d} \\ &  =
      \gamma {(n')}^{-1/d} =  r_{n'}.
      \end{align*}
      Additionally,
      \begin{align*}
      {r_{n}^s}^{\rrg} & =(1+\mu)r_n^{st}\\ &  = (1+\mu)\frac{\beta \log^{1/(d-1)}n}{n^{1/d}} \\ & = \frac{\beta \log^{1/(d-1)}n}{((1+\mu)^{-d}n)^{1/d}} \\ & \geq  \frac{\beta \log^{1/(d-1)}((1+\mu)^{-d}n)}{((1+\mu)^{-d}n)^{1/d}} = {r_{n'}^s}^{\rrg}.
      \end{align*}
      This implies that all samples added after iteration $n-n'$ are
      connected in a \prm-fashion with a radius $r_{n'}$. Thus,
      $\R_n\cap \B_{\kappa/2}(\pi)$ contains a \prm graph
      $\P_{n'}\cap \B_{\kappa/2}(\pi)$, constructed with the radius
      $r_{n'}$ with probability $1-a\cdot e^{-b'n}$, as required. 
	\end{proof}
	
	We are now ready to finish the proof of Theorem~\ref{thm:rrg}.  Let us
	denote by $A$ the event that the graph
	$\R_n$ contains a \prm graph
	$\P_{n'} \bigcap \B_{\kappa/2}(\pi_\eps)$.
	Given that $A$ occurs, we denote by $B$ the event that the constructed
	roadmap $\R_n \bigcap \B_{\kappa/2}(\pi_\eps)$ maintains the good
	properties of \prm, as defined in Theorem~\ref{thm:prm}.ii.1. That
	is, $\P_{n'} \bigcap \B_{\kappa/2}(\pi_\eps)$ contains a path
	$\pi_{n'} \in \Pi_{s,\X_\text{goal}}^\F$ with
	$\len(\pi_{n'})\leq(1+\eps)\xi\cdot\lenopt$;
	
	Given that $A$ occurs, from Theorem~\ref{thm:prm} we obtain that $B$
	occurs on the \prm-like graph inside $\B_{\kappa/2}(\pi)$ with
	probability of at least
    \begin{align*}
	1-O\left((n')^{-1}\right) & = 1- O\left((1+\mu)^{-d}n^{-1}\right) \\ &= 
	1-O\left(n^{-1}\right).
    \end{align*}

	That is, $\Pr[B|A] \geq 1-O({n}^{-1})$.  Additionally,
	Lemma~\ref{lem:rrg_contains_prm} states that
	$\Pr[A]\geq 1-ae^{-b'n}$ for some $b'$.
	Note that we would like to bound the probability that both $A$ and $B$
	occur simultaneously. From the definition of conditional probability
	we have that $\Pr[A\cap B] = \Pr[B|A]\cdot \Pr[A]$.  Therefore,
    \begin{align*}
	\Pr[A\cap B] & \geq
	\left(1-O(n^{-1})\right)\cdot\left(1-ae^{-b'n}\right) \\ &\geq 1-
	O(e^{-b'n}) - O(n^{-1}) \\ & = 1- O(n^{-1}).
    \end{align*}

	\subsection{Multi-robot planners}
	In this section we briefly state the implications of our analysis
	of \prm for sampling-based multi-robot motion planning. In
	particular, we consider the planners \mstar~\citep{WagCho15},
	\drrt~\citep{SolETAL16drrt} and \drrtstar~\citep{DobETAL17}.
	
	The multi-robot setting involves $m\geq 2$ robots operating in a
	shared workspace.  In this more challenging setting collisions
	between different robots must be avoided, in addition to standard
	robot-obstacle collisions. Single-robot sampling-based planners
	can be applied directly to the problem by considering the robot
	fleet as one highly-complex robot. However, recently-introduced
	planners that are tailored for the multi-robot case have proved
	to be much more effective than the aforementioned naive approach.
	
	Such recent approaches include~\mstar, \drrt, and \drrtstar. A
	common ground to these techniques is an implicit construction of
	a composite roadmap~$\mathbb{G}$~\citep{SolETAL16drrt}, which
	results from a tensor product between $m$ single-robot roadmaps
	$G_1,\ldots,G_m$---for every $1\leq i\leq m$ the graph $G_i$ is a
	\prm embedded in the configuration space of robot $i$. It was
	recently proved that if $G_1,\ldots,G_m$ are AO, with respect to
	the individual robots for which they are defined, then
	$\mathbb{G}$ is AO with respect to the multi-robot
	problem~\citep{DobETAL17}. 
	
	Thus, by Theorem~\ref{thm:prm} it follows that the underlying
	single-robot \prm graphs $G_1,\ldots,G_m$ can be constructed
	using a smaller connection radius, while still guaranteeing the
	AO (or AnO) of $\mathbb{G}$. This in turn, guarantees AO (or AnO) of
	\mstar, \drrtstar, and PC of \drrt.

	\section{Experimental results}
	\label{sec:experimental_results}
	We present experiments demonstrating the effect of using different
    values of the connection radius $r_n$ on the running time, cost of
    solution for the length cost $\len$, and success rate, when
    running the algorithms \prm and \fmt on problems of dimensions up
    to $12$. In particular, $r_n$ ranges between the critical radius
    (Theorem~\ref{thm:prm}) and previously obtained upper bounds
    from~\cite{JSCP15,KF11}.
	
	We validate our theory for \prm (Theorem\ref{thm:prm}) and \fmt
    (Corollary~\ref{cor:fmt}).  We observe that smaller connection
    radii, than previously-obtained bounds, still allow the planners
    to converge to high-quality, near-optimal paths.
     Furthermore, we identify situations in which using a radius that
     is close to $r^*_n$ allows to obtain a high-quality solution more
     quickly.  Moreover, although the resulting cost for the smaller
     radii can be slightly worse, we observe that postprocessing the
     paths using standard simplification methods yields solutions that
     are only marginally inferior to the best (postprocessed) solution. 
     Specifically, in harder
     scenarios the advantage in using a smaller connection radius is
     more prominent; in some cases we obtain a reduction of 50\% in
     running time, with an improved cost, and similar success rates when
     compared to the results obtained using the original \fmt
     connection radius.  
	\subsection{Implementation details}
	In our experiments, we used the Open Motion Planning Library (OMPL
	1.3, \cite{OMPL}) on a 2.6GHz$\times 2$ Intel Core i5 processor with
	16GB of memory.  Results were averaged over 50 runs and computed for
	dimensions up to 12. 
	
		The planners that we used are \prm and the batch variant of \fmt,
		which were adapted for samples from a PPP, where $n$ is the expected number of samples.  Specifically, given
		the expected number of samples $n$, these variants generate a set
		of samples according to the recipe in Claim~\ref{clm:recipe}.  
		The
		two planners use the connection radii $r_n,r^{st}_n$.  Note that $r^{st}_n$ should be at least
		$\frac{\beta\log^{1/(d-1)}n}{n^{1/d}}$, but the exact value of
		$\beta$ is unknown.  For simplicity, 
		we set 
		$r^{st}_n$ to be identical to $r_{\prm^*}$, defined
		in~\cite{KF11}. We emphasize that although we use an asymptotically smaller value,
		it still yields
		(empirically) convergence in cost.  This suggests that the bound
		on $r^{st}_n$ can be further reduced.
	
    \begin{wrapfigure}{r}{0.2\columnwidth}
		\vspace{-10pt}
		\centering 
		\includegraphics[trim={7cm 5.7cm 6.5cm 5.7cm},clip, width=0.2\columnwidth]{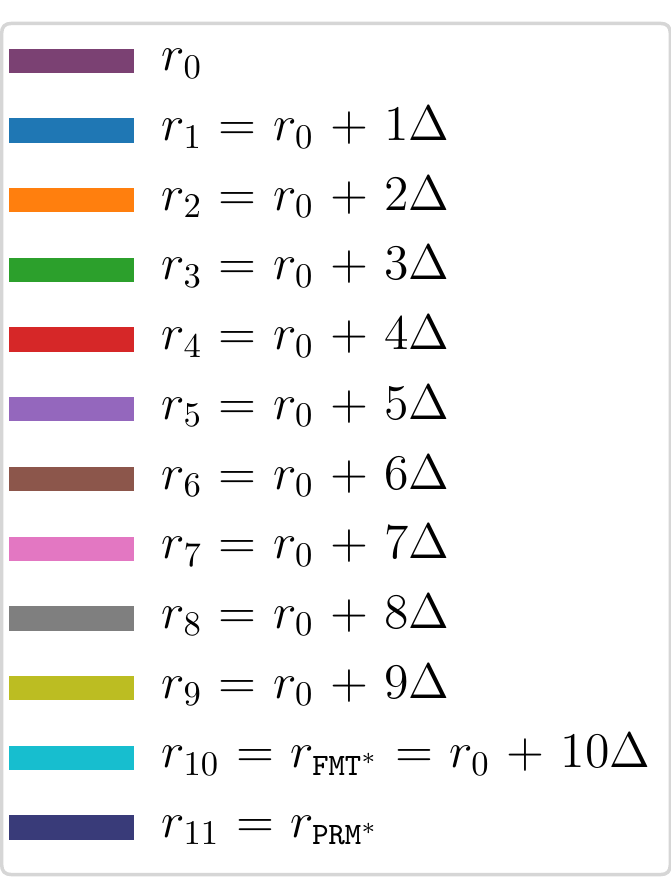}
		\vspace{-5pt}
	\end{wrapfigure}
	Given a scenario and a value $n$, we define a set of $k+1$
    increasing connection radii, $\{r_0, \ldots , r_k\}$, as follows.
    We set the minimal connection radius to be
    $r_0 = \gamma n^{-1/d}$, where $\gamma=1$. Note that $\gamma$ is
    larger than $\gamma^*$ by a factor of roughly $2$. The maximal
    connection radius, denoted by $r_k = r_{\text{\fmt}}$, is as
    defined in~\cite{JSCP15}.  For each $1\leq i\leq k-1$ we define
    $r_i = r_0 + i\cdot\Delta$, where $\Delta=(r_k-r_0)/k$.  Now, for
    every scenario and number of samples $n$ we run our planning
	algorithm with $r^{st}_n$, and $r_n\in\{r_0, \ldots , r_k\}$.
    Note that all our plots are for $k=10$, and that in some
    experiments an additional radius
    $r_{k+1} = r_{\prm^*}>r_{\text{\fmt}}$ appears as well
    (see~\cite{KF11}). The figure to the right depicts the colors and
    labeling that will be used throughout this section.

	\begin{figure*}
\centering
		\begin{subfigure}[b]{.3\linewidth}
          \centering
			\includegraphics[height=4cm]{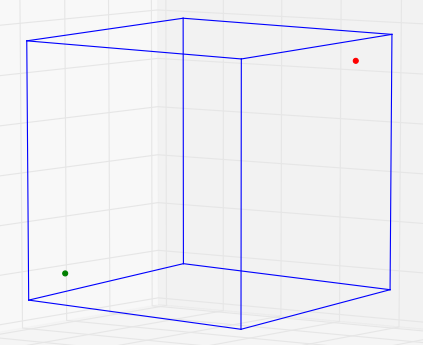}
            \caption{}
			\label{fig:hypercube}
		\end{subfigure}%
		\begin{subfigure}[b]{.3\linewidth}
          \centering
			\includegraphics[height=4cm]{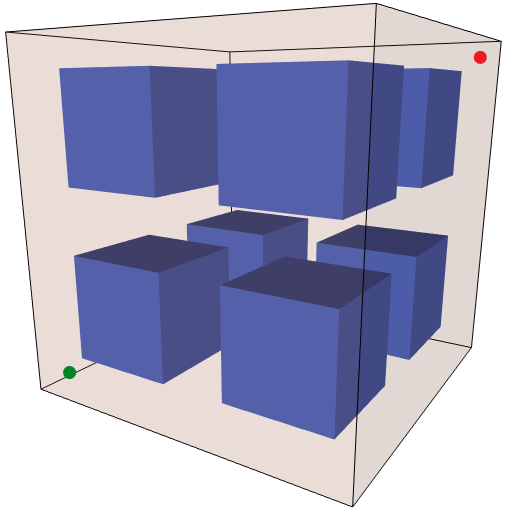}
             \caption{}
			\label{fig:hypercube_with_obst} 
		\end{subfigure}
		\begin{subfigure}[b]{.3\linewidth}
\centering
			\includegraphics[height=4.5cm]{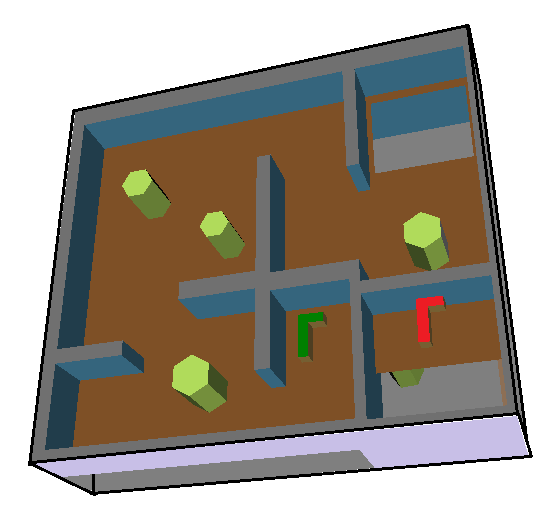}
 \caption{}
			\label{fig:cubicles} 
		\end{subfigure}
		\caption{Scenarios used in experiments. (a)~$d$D Hypercube,
			(b)~$d$D Hypercube with $2^d$ hypercubical obstacles, and
			(c)~3D Cubicles.  Start and target configurations for a
			robot are depicted in green and red, respectively.  Scenario
			(c) is provided with the~OMPL distribution.}
		\label{fig:scenarios}
	\end{figure*}

	\subsection{Results}
	\noindent\textbf{Euclidean space.}
	The scenario we consider (see Figure~\ref{fig:hypercube}) consists of a point robot moving in
	the obstacle-free unit $d$-dimensional hypercube. Therefore,
	$\F = \C = [0,1]^d$.  We set the start and target positions of the
	robot to be $s=(0.1,\ldots, 0.1)$ and $t=(0.9,\ldots, 0.9)$,
	respectively. 
	
	We run \prm and plot (i) the overall running time, (ii) the normalized cost ($\len$) of the obtained
	solution, where a value of 1 represents the best possible cost, and (iii)
	the portion of successful runs---all as a function of the expected
	number of samples $n$.  Results are
	depicted in Figure~\ref{fig:prm_empty}.
	
	
	The plots demonstrate the following trend: for each radius~$r$ the
	cost obtained by \prm converges to some constant times the optimal cost, which is marked with the dashed
	red curve in the ``Cost vs. $n$'' plot. (We note
		that it is possible that for larger values of $n$ the cost
		values for different radii will eventually converge to the same
		value.) 
	Clearly, $r_{\prm^*}$ yields the best cost but at the price of
	increased running time.  $r_{10} = r_{\text{\fmt}}$ obtains the
	next best cost, with improved running time, and so on.  Note that
	for $d=4$, already for $n=1$K a solution is found for all radii
	except $r_0$, whereas for $d=8$ and $n=5$K a solution is found for
	all radii above $r_4$.  It is important to note that there is a
	clear speedup in the running times of \prm when using $r_i$ for
	$i\leq 9$, over $r_{\text{\fmt}}, r_{\prm^*}$, with a slight
	penalty (a factor of roughly $2$) in the resulting costs.

    For the first set of experiments of \prm in an obstacle free
$d$-dimensional hypercube, we also study how the value of $r_i$
affects the size of the connected components. Note that
  Theorem~1 in \cite{PenPiz96} states that $|C_n|=\Theta(n)$. In
particular, we measure the size of the two largest connected
components, denoted by $C_n$ and $C'_n$, as a function of both the
radius $r_i$ and the expected number of samples $n$. The results for
$d=2, d=12$ are summarized in Table~\ref{tbl:size_of_cc2}.  Already for
$r_2$, $C_n$ is significantly larger than $C'_n$.  However, for $r_0$
there is no clear difference. That is, we do not see in practice the
expected emergence of the ``huge'' component.  As $n$ increases, the
proportion of $C_n$ increases as well, whereas that of $C'_n$ shrinks.
Also note that for specific $r_i, n$ the maximal component is smaller
for $d=12$ than for $d=2$.


\begin{table*}[h]
		\centering
		\begin{tabular}[]{c c ||c | c | c | c | c | c }
			& & \multicolumn{2}{c|}{$r_0$} & \multicolumn{2}{c|}{$r_2$} & \multicolumn{2}{c}{$r_{10}$} \\ \hline
			$d$ & $n$ & $|C_n|/n$ & $|C'_n|/n$ & $|C_n|/n$ & $|C'_n|/n$ & $|C_n|/n$ & $|C'_n|/n$ \\ \hline \hline
			2 & 1K & 0.17 & 0.1 & 0.88 & 0.05 & 1.0 & 0.0 \\ \hline
			2 & 5K & 0.08 & 0.05 & 0.97 & 0.001 & 0.999 & 0.0 \\ \hline
			2 & 10K & 0.05 & 0.04 & 0.98 & 0.003 & 1.0 & 0.0 \\ \hline
			2 & 50K & 0.02 & 0.01 & 0.99 & 0.001 & 1.0 & 0.0 \\ \hline
			\hline
			12 & 1K & 0.04 & 0.02 & 0.84 & 0.003 & 1.0 & 0.0 \\ \hline
			12 & 5K & 0.12 & 0.01 & 0.91 & 0.001 & 1.0 & 0.0 \\ \hline
			12 & 10K & 0.21 & 0.005 & 0.94 & 0.001 & 1.0 & 0.0 \\ \hline	
		\end{tabular}
		\caption{The average proportion of the two largest connected components $C_n,C'_n$ obtained using \prm  as a function of $r_i, n, d$.}
        \label{tbl:size_of_cc2}
	\end{table*}

	\begin{figure*}
      \vspace{-10pt}
		\centering 
		\begin{subfigure}{.32\textwidth}
				\includegraphics[trim={0 0 0 2.5cm},clip,width=1\textwidth]{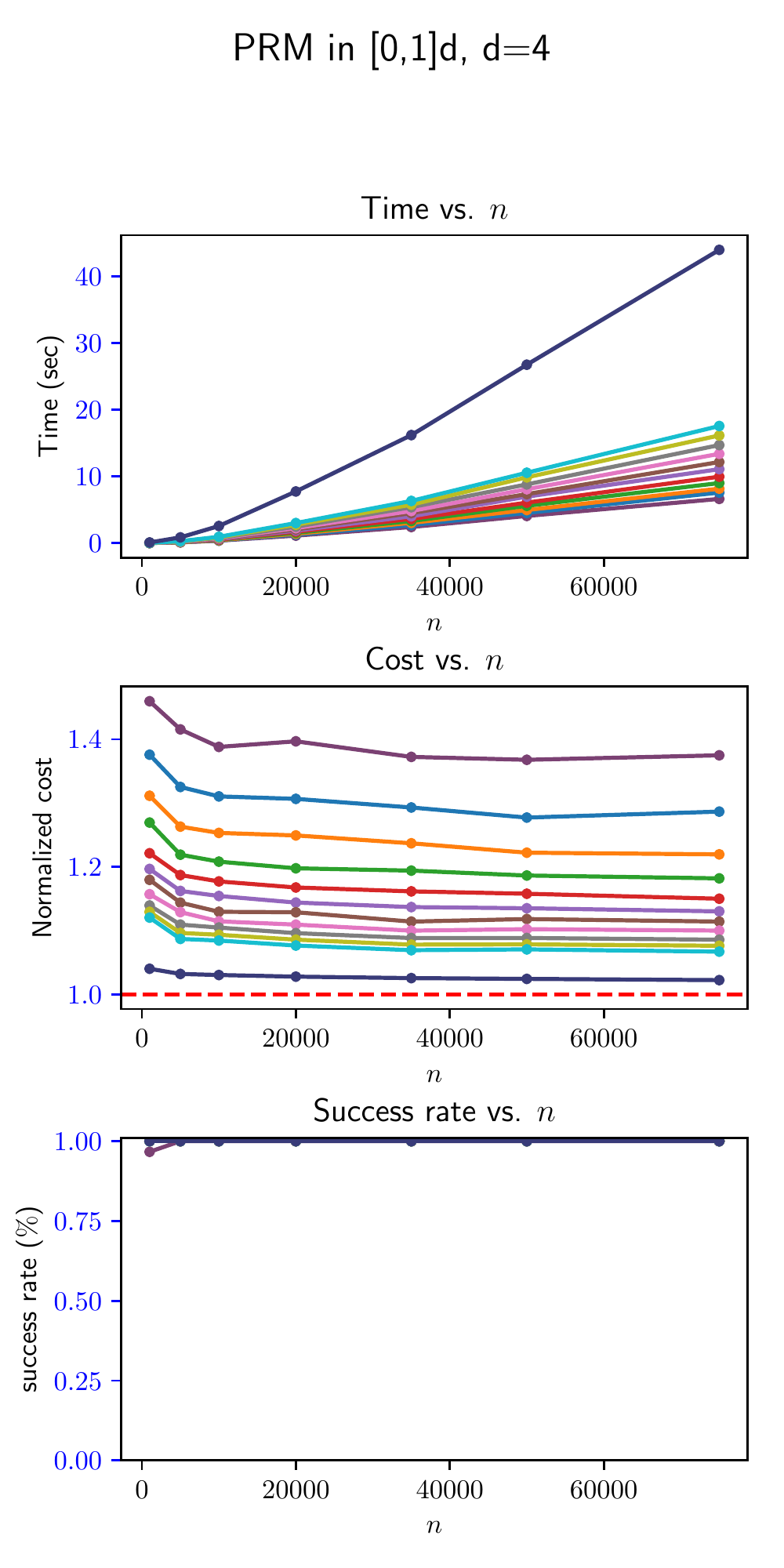}
				\caption{$d=4$}
                \label{fig:prm_res_4d}
		\end{subfigure}
		\begin{subfigure}{.32\textwidth}
			\includegraphics[trim={0 0 0 2.5cm},clip,width=1\textwidth]{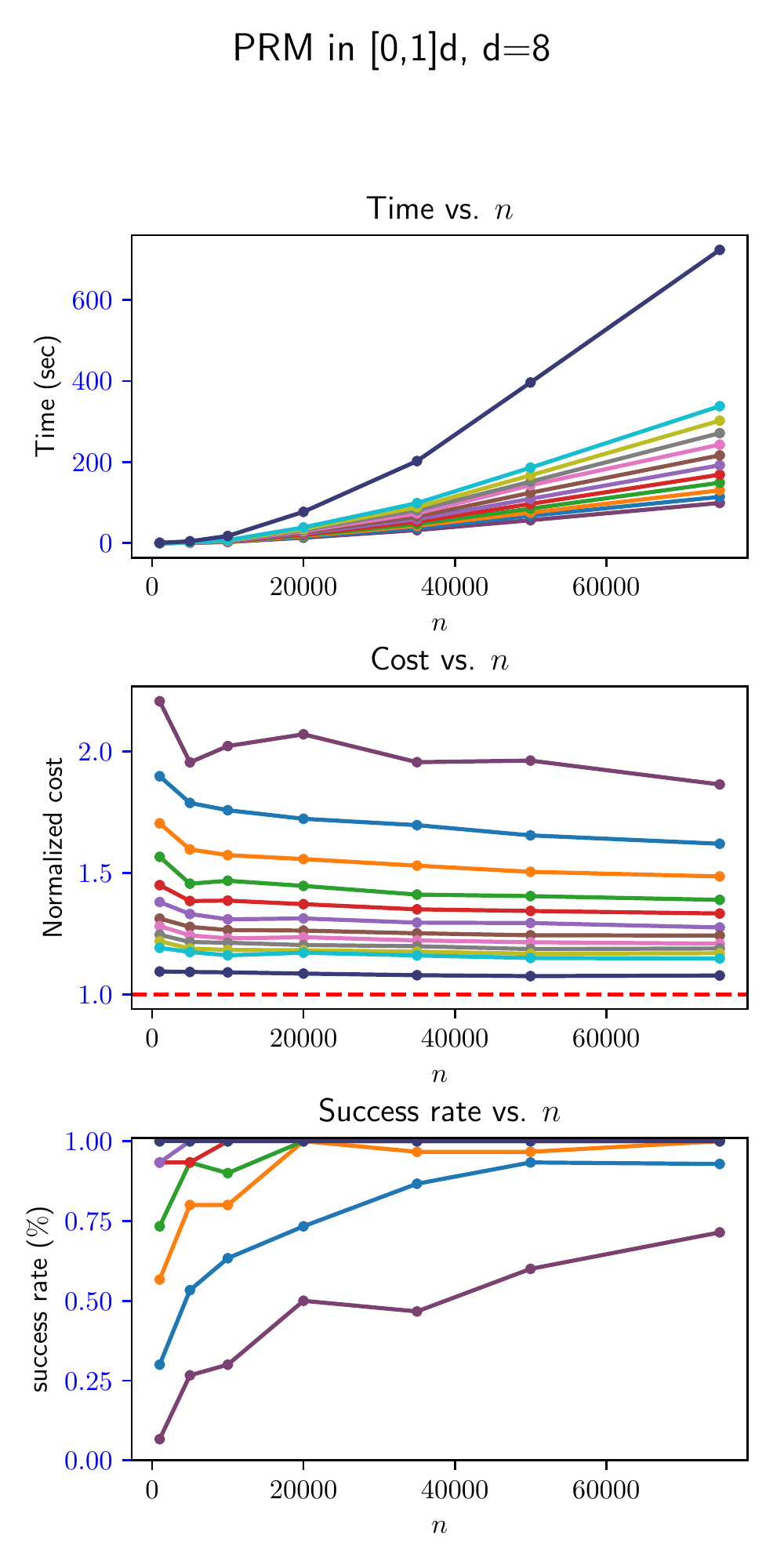}
            \caption{$d=8$}
			\label{fig:prm_res_8d} 
		\end{subfigure}
        \begin{subfigure}{.32\textwidth}
			\includegraphics[trim={0 0 0 2.5cm},clip,width=1\textwidth]{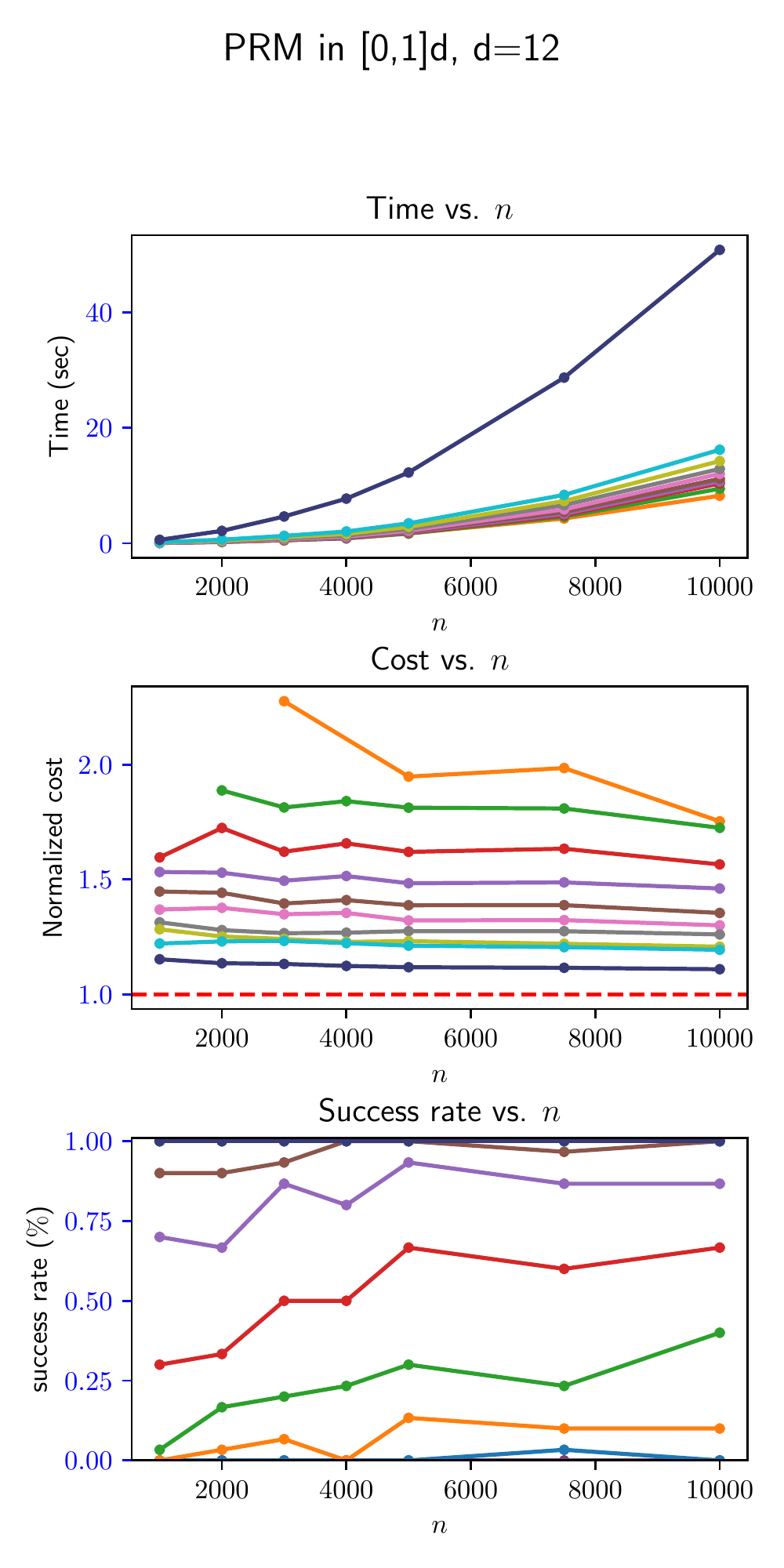}
            \caption{$d=12$}
			\label{fig:prm_12} 
		\end{subfigure}
		\caption{\prm in the (a) $4$D, (b) $8$D, and $12$D obstacle free hypercube.\label{fig:prm_empty}}
	\end{figure*}	
    \vspace{5pt}

	\noindent\textbf{General Euclidean space.}
	We consider the following $d$-dimensional scenario (based on a
    scenario from~\cite{SolETAL16}), for $d\in \{4,8\}$, depicted in
    Figure~\ref{fig:hypercube_with_obst} in 3D, and a point robot: $\C = [0,1]^d$
    is subdivided into $2^d$ sub-cubes by halving it along each axis.
    Each sub-cube contains a centered $d$-dimensional axis-aligned
    hypercubical obstacle that covers 25\% of the sub-cube.  The start
    position $s$ of the robot is placed on the diagonal between
    $(0,\ldots, 0)$ and $(1,\ldots, 1)$, such that it is equidistant
    from the origin and from the closest hypercubical obstacle.  $t$
    is selected similarly, with respect to $(1,\ldots, 1)$.
	
	Here we use \fmt with radii ranging from $r_0$ to $r_{10}$. We
	also ran \prm which exhibited a similar behavior.
	Figure~\ref{fig:fmt_hypercube_with_obst_4d} presents the results
	for $d=4$. We plot the average cost after simplifying the
	resulting paths in addition to the average original cost.  We
	mention that there is a difference in cost between the various
	radii and that costs obtained using larger radii are often better.
	However, after applying OMPL's default path-simplification
	procedure we obtain paths with negligible differences in cost.
	This suggests that paths obtained using the smaller
	connection radii are of the same homotopy class as the path
	obtained using $r_{\text{\fmt}}$.

For $d=8$ (Figure~\ref{fig:fmt_hc_with_obs_8d}), as in the case of $d=4$, costs obtained using larger radii are often better.
However, applying OMPL's default path-simplification
procedure yields paths with negligible differences in cost.
We do mention, though, that the success rates of all radii deteriorate, when compared to $d=4$.

%
	\vspace{5pt}\noindent\textbf{General non-Euclidean space.}
	We use the Cubicles scenario (see Figure~\ref{fig:cubicles})
	provided with the OMPL distribution~\citep{OMPL}.  Here the goal is
	to find a collision-free path for one or 
	two L-shaped robots that need to exchange their positions.
    In the single-robot case, the robot needs to move between the two configurations depicted in red and green. In the two-robot case, these two configurations represent the start positions of the robots, where the goal is for each robot to finish at the start of the other robot.  Since the robots are allowed to
	translate and rotate, then $d=6$ or
	$d=12$, respectively, and $\C$ is non-Euclidean.  Although our theory does not
	directly apply here, we chose to test it empirically for such a
	setting.  	

	We ran our experiments with \fmt using a connection radii ranging
	from $r_0$ to $r_{10}= r_{\text{\fmt}}$. Initially, no solution
	was found in all runs.  Indeed, as was mentioned in~\cite{KSH16},
	this is not surprising, since the theory from which the radius
	values are derived assumes that the configuration space is
	Euclidean. In the same paper the authors propose a heuristic for
	effectively using radius-based connections in motion planning. In
	our experiments (Figure~\ref{fig:fmt_cubicles_12d_r}) we increased
	all radii by a multiplicative factor of $3$ and $10$, respectively, in order to increase
	the success rates. We mention that this yielded similar behavior
	to that when using the connection scheme suggested
	in~\cite{KSH16}.

    Results for the single-robot case are depicted
    in~\ref{fig:fmt_cubicles_1robot}. In this case, even with 25K
    samples already with $r_3$ (green) we were able to obtain high
    successes rates.  The running time was slightly better than that
    of $r_{10}$, whereas the cost was higher than the one obtained
    using $r_{10}$.  However, after applying the smoothing procedure
    the difference in cost decreased significantly.

	Let us proceed to the two-robot setting. Observe in Figure~\ref{fig:fmt_cubicles_12d_r}, that the variance
	 in cost after path simplification is again significantly smaller
	 than the variance in the original cost.  Clearly, smaller radii
	 exhibit shorter running times, but with smaller success rates.
	 However, since the success rate improves as the number of samples
	 $n$ increases, one could use the smaller radii with larger $n$ and
	 still obtain a solution with comparable cost and success rate in
	 shorter running time.  Indeed, using $r_3$ (green) with 75K
	 samples we obtain a solution whose cost (after simplification) is
	 slightly better than the cost obtained using $r_{10}$ with 42K
	 samples (after simplification).  Moreover, the running time of the
	 former is roughly 50\% of the time taken for the latter, and both
	 obtain similar success rates.  This indicates that in such
	 settings, one could benefit from using smaller radii in terms of
	 favorable running times and obtain a comparable or even better
	 cost with similar success rates.	

	\begin{figure*}
		\vspace{-10pt}
		\centering 
		\begin{subfigure}[b]{.35\textwidth}
				\includegraphics[trim={0 0 0 2.5cm},clip, width=1\textwidth]{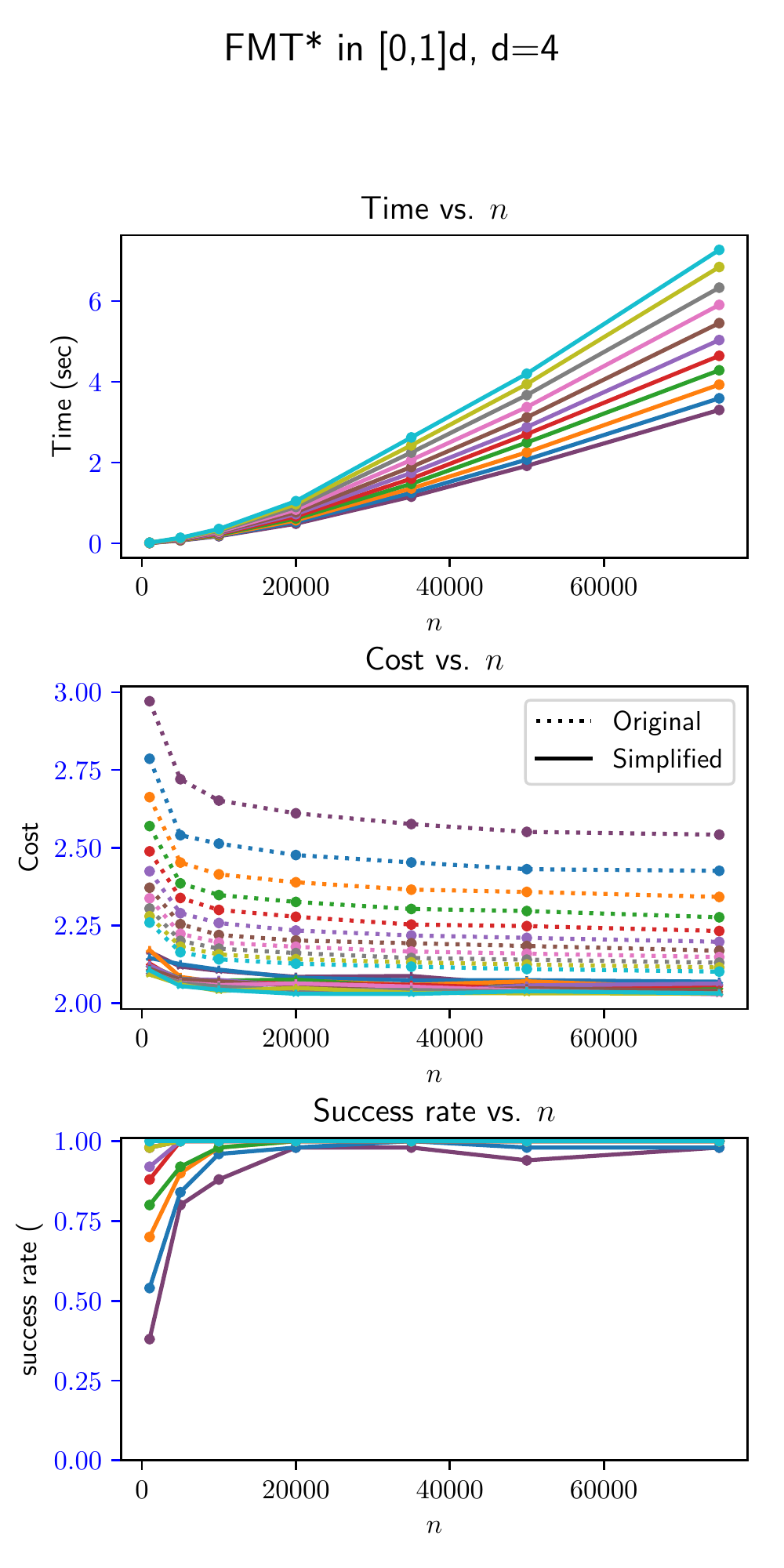}
\caption{$d=4$}
				\label{fig:fmt_hypercube_with_obst_4d}       
		\end{subfigure}
        \begin{subfigure}[b]{.35\textwidth}
			\includegraphics[trim={0 0 0 2.5cm},clip,width=1\textwidth]{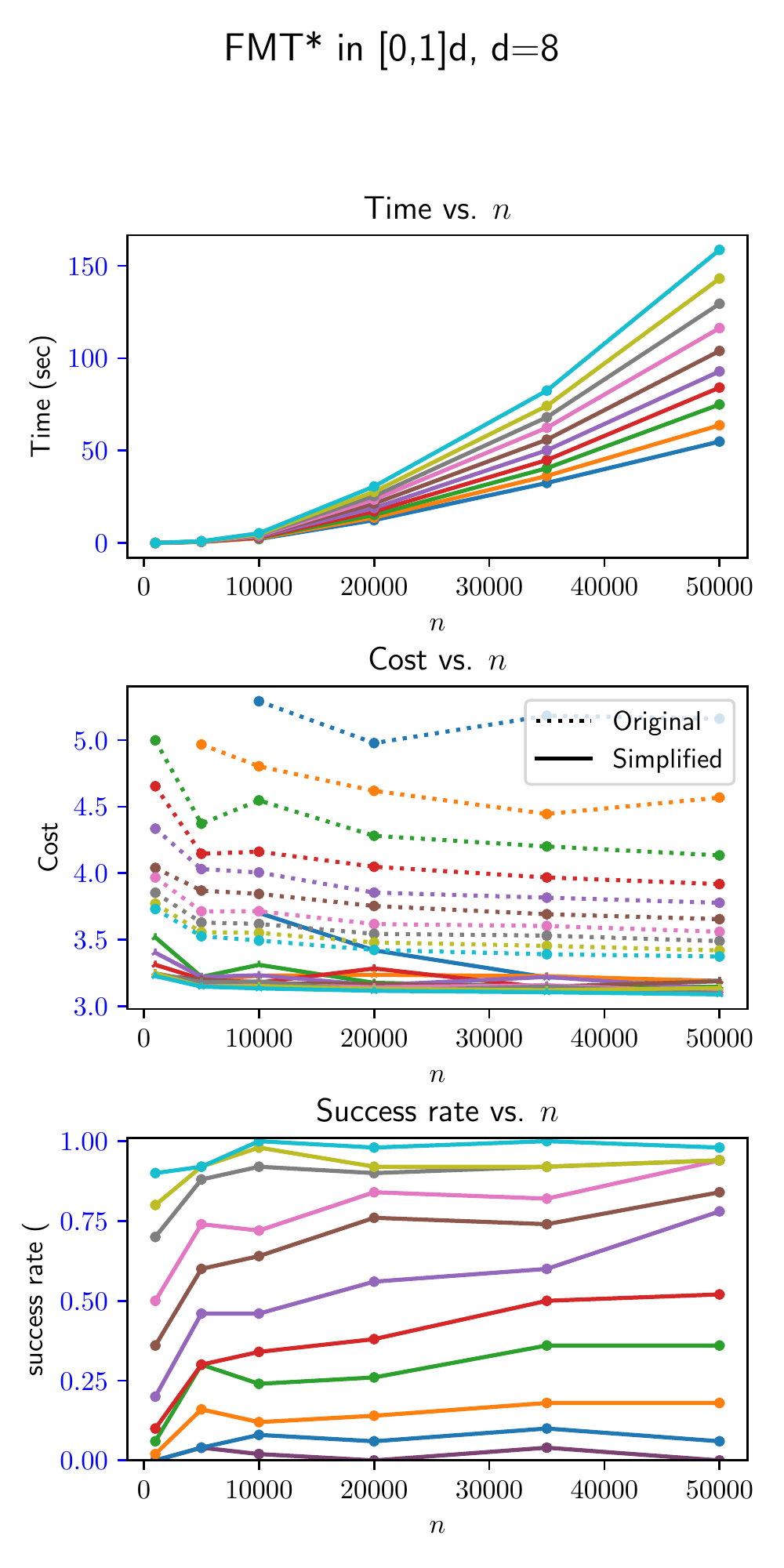}
			\caption{$d=8$}
			\label{fig:fmt_hc_with_obs_8d} 
		\end{subfigure}
		\caption{\fmt for a point robot in the (a) $4$D and (b) $8$D
          hypercube with obstacles scenario
          (Figure~\ref{fig:hypercube_with_obst}). Both the average
          original cost and the average cost after simplification are
          presented for each radius.  There is a difference in cost
          between the various radii (dashed lines), that diminishes
          after simplifying the resulting paths (solid lines).  }
		\label{fig:fmt_cubihcles} 
				\vspace{-10pt}
	\end{figure*}
	
      \begin{figure*}
\vspace{-10pt}
		\centering 
		\begin{subfigure}[b]{.35\textwidth}
			\includegraphics[trim={0 0 0 2.5cm},clip, width=1\textwidth]{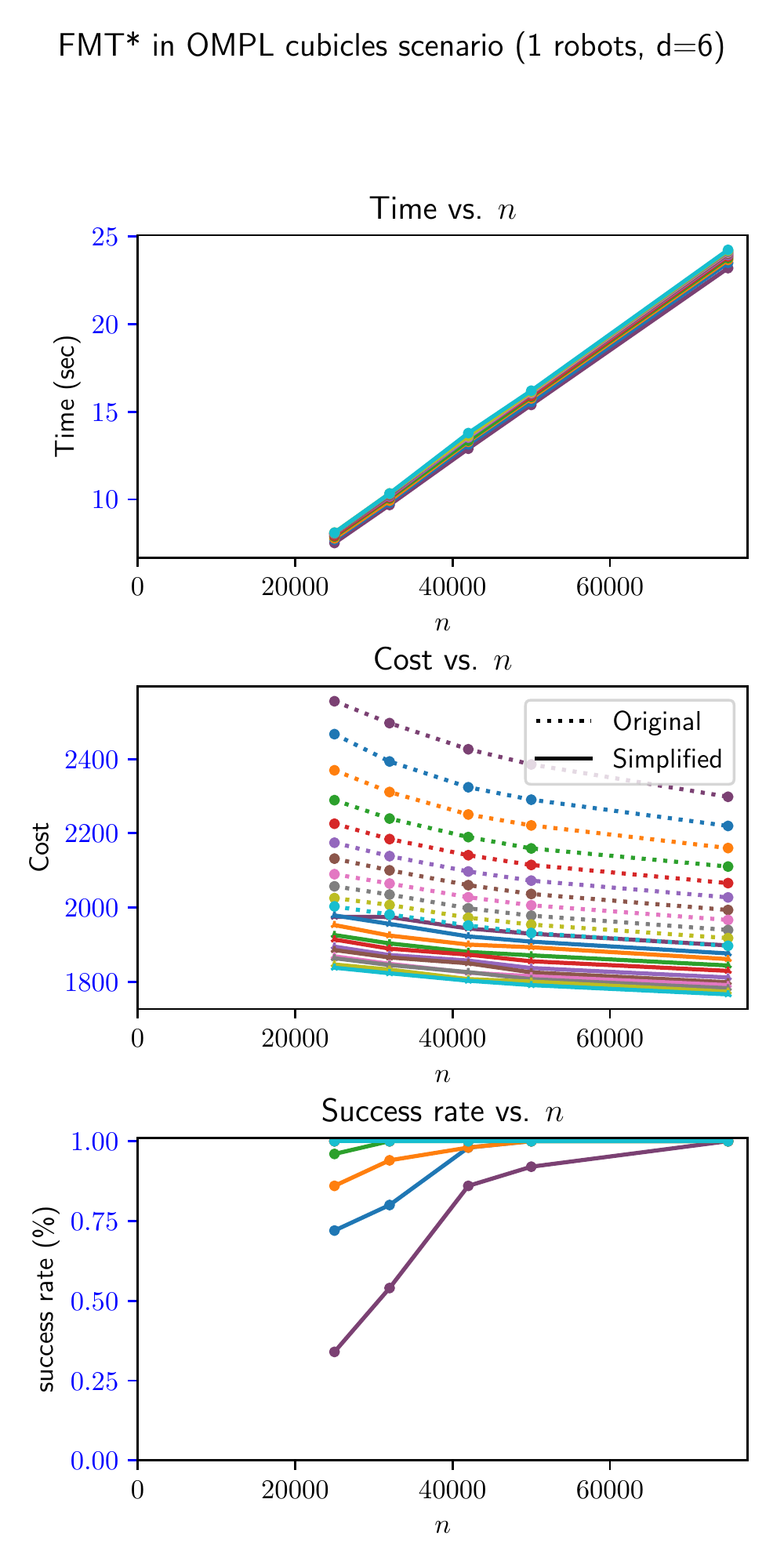}
			\caption{$d=6$}
			\label{fig:fmt_cubicles_1robot} 
		\end{subfigure}
        	\begin{subfigure}[b]{.35\textwidth}
			\includegraphics[trim={0 0 0 2.5cm},clip,width=1\textwidth]{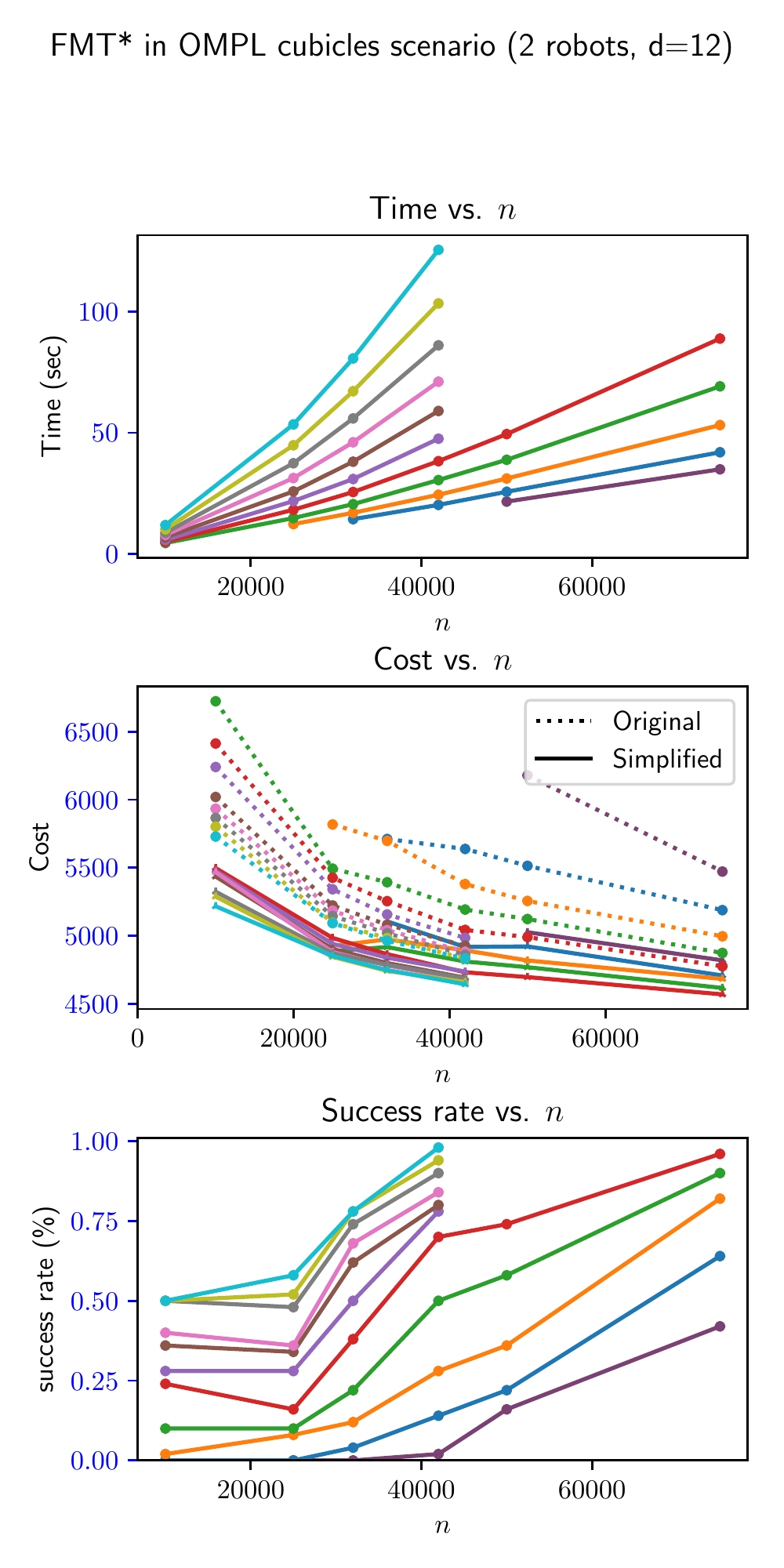}
            \caption{$d=12$}
			\label{fig:fmt_cubicles_12d_r} 
		\end{subfigure}
		\caption{ \fmt for (a) one and (b) two rigid-body robot(s) moving in OMPL's Cubicles scenario (Figure~\ref{fig:cubicles}).
          				Both the average original cost and the average cost after simplification are presented for each radius.}
		\label{fig:fmt_additional_plots} 
        \vspace{-10pt}
	\end{figure*}

	\section{Future work}\label{sec:future}
	In this work we leveraged techniques from percolation theory to
    develop stronger analysis of \prm-based
    sampling-based motion planners. In the hope of providing
    mathematically rigorous presentation, while still being
    accessible, we chose to focus the discussion on simplified,
    possibly unrealistic, robotic systems. In particular, we assumed a
    holonomic system having a Euclidean configuration space.

    Our immediate future goal is to extend our theory to non-Euclidean
    spaces, such as those arising from rigid bodies translating and
    rotating in space. The next challenge will be to extend the model
    to robots with differential constraints. While the latter task
    seems daunting, we should keep in mind that the state space of
    such systems can be modeled as a differential manifold. This may
    allow to locally apply our Euclidean-space techniques to analyze
    manifold spaces. Indeed, a similar approach has already been
    considered in previous work~\citep{SchETAL15, SchETAL15b}.

    The next algorithmic challenge is to consider robotic systems for
    which precise \emph{steering}, i.e., solving the \emph{two-point
      boundary value problem}, cannot be performed, at least not
    efficiently (see discussion in~\cite[Section~1.3]{LiETAL16}). In
    such cases, \prm-based planners (as we considered here), or
    \rrtstar-based techniques, cannot be applied. We pose the
    following question: Is it possible to extend existing planners to
    work in the absence of a steering function, while maintaining
    their theoretical properties? We plan to tackle this question in the near future. 

	
\section{Acknowledgments}
	We thank the participants of the \emph{Workshop on Random Geometric
		Graphs and Their Applications to Complex Networks}, 2015, and
	in particular Carl P.\ Dettmann, Mathew Penrose, Paul Balister, and
	Yuval Peres, for fruitful discussions. We also thank Dan Halperin for
	commenting on early drafts of the paper.

    This work has been supported in part by the Israel Science
    Foundation (grant no.~825/15), by the Blavatnik Computer Science
    Research Fund, and by the Blavatnik Interdisciplinary Cyber
    Research Center at Tel Aviv University. M.K.\ is supported in part
    by the Yandex Foundation.  This work was prepared as part of
    K.S.'s Ph.D.\ thesis in Tel Aviv University, with the support of
    the Clore Israel Foundation.


\appendix

\section*{Appendix}

	\section{Theory for (semi-)deterministic sampling}\label{sec:deterministic}
	So far we have studied the theoretical properties of various
    sampling-based planners constructed with the PPP sample set
    $\X_n$.  In this section we study similar aspects but for a
    different sampling regime. Particularly, we replace $\X_n$ with
    $\L_n$ which represents a rectilinear lattice, also known as a
    Sukharev sequence. A recent paper by \cite{LucETAL17} studies the
    theoretical properties of \prm and \fmt when constructed with
    $\L_n$ (and other deterministic low-dispersion sequences).  Here
    we introduce a new analysis which shows that \prm constructed with
    $\L_n$ can be further sparsified while maintaining AO. This result
    can also be extended to \fmt, \btt, \rrg and other planners, as
    was done in Section~\ref{sup:other_planners} for the setting of
    PPP.
	
	\subsection{The basics}
	We start with some basic definitions. For a given $n$ define
	$\tilde{r}^*_n=n^{-1/d}$ and $\L_n=\tilde{r}^*_n\cdot
	\mathbb{Z}_d$.
	That is, $\L_n$ contains all the points of a rectangular lattice, with
	edge length $\tilde{r}^*_n$. Observe that $\G(\L_n;\tilde{r}^*_n)$ is
	a deterministic graph which connects by an edge every two points
	$x,x'\in \L_n$ that differ in exactly one coordinate and the distance
	between them is exactly $\tilde{r}^*_n$. We proceed to define a
	randomized subset of $\G(\L_n;\tilde{r}^*_n)$:
	
	\begin{definition}
      Given $0\leq p\leq 1$, the graph
      $\G_n^p=\G(\L_n;\tilde{r}^*_n;p)$ is defined in the following
      manner. Every point of $\L_n$ is added independently with
      probability $p$ to $\G_n^p$ as a vertex. Two vertices
      $x,x'\in \G_n^p$ are connected by an edge if
      $\|x-x'\|\leq \tilde{r}^*_n$.
	\end{definition}
	
	This corresponds to a well-studied model of \emph{site
      percolation}.  Similarly to the continuous model studied in the
    previous section, we say that $\G_n^p$ \emph{percolates} if it
    contains an infinite connected component $C_\infty^p$ that
    includes the origin~$o$. Also denote by $C_n^p$ the largest
    connected component of $[0,1]^d\cap C_\infty^p$.
	
	The following theorem corresponds to a combination of
    Theorem~\ref{thm:radius}, Lemma~\ref{lem:psi}, and
    Theorem~\ref{thm:unique} for the continuous case.
	
	\begin{theorem}{\cite{FraMee08}}
		There exists a critical value $0<p^*<1$ such that for $p<p^*$
		the graph $\G_n^p$ percolates with probability
		zero. If $p>p^*$ then $\G_n^p$ percolates with
		non-zero probability. Furthermore, if $p>p^*$ then $C^p_\infty$ is unique with probability~$1$.
	\end{theorem}
	
	The next two lemmata serve as  discrete versions of Lemma~\ref{lem:H_n} and Lemma~\ref{lem:H'_n}, respectively. Their proofs are very similar to the original proofs and are therefore omitted. We do mention that they rely on Theorem~5 and Theorem~4 of~\cite{PenPiz96}, respectively.
	
	\begin{lemma}\label{lem:H_n_p}
		Suppose that $p>p^*$.  
		Denote by $\EE'^1_n$ the event that $C_{\infty}^p\cap H_n\subset C_n^p$,
		where $H_n$ is as defined in Lemma~\ref{lem:H_n}.  Then there exists
		$n_0\in \dN$ and $\alpha>0$ such that for any $n>n_0$ it holds that
		\[\Pr[\EE'^1_n]\geq 1-\exp\left(-\alpha n^{1/d}\log^{-1}n\right).\] 
	\end{lemma}
	
	\begin{lemma}\label{lem:H'_n_p}
		Suppose that $p>p^*$. Define $H''_n\subset H_n$ to be a hypercube of
		side length
		\[h''_n=\frac{\beta'\log^{1/(d-1)}n}{n^{1/d}}.\]
		Denote by $\EE'^2_n$ the event that $H''_n\cap C^p_n\neq \emptyset$.
		Then there exist $n_0\in \dN$ and $\beta'_0>0$ such that for
		any $n>n_0,\beta'>\beta'_0$ it holds that
		$\Pr[\EE'^2_n | \EE'^1_n]\geq 1-n^{-1}$.
	\end{lemma}
	
	Now we proceed to the discrete counterpart of
	Theorem~\ref{thm:stretch}. 
	For any $v\in \dR^d$ define
	$x(v)=\textup{argmin}_{x\in \L_n}\|x-v\|$. The following is a
	simplified version of~\cite[Theorem~1.1]{AntPis96}.
	
	\begin{theorem}\label{thm:strech_discrete}
		Suppose that $p>p^*$. There exists a positive constant $\xi'\geq
		1$ such that for any two fixed points $v,v'\in \dR^d$ it
		holds that 
        \begin{align*}
		\Pr[x(v),x(v')&\in \C_\infty, \dist\left(\G_n^p, x(v),x(v')\right) \\& >\xi'\|x(v)-x(v')\|]=o(1).
        \end{align*}

	\end{theorem}
	
	For the subcritical regime ($p<p^*$), we have the following
	theorem. Define 
	\[S_k=\left\{x\in \L_n \middle| \frac{\|x\|_1}{\tilde{r}^*_n}\leq
	k\right\}.\] 
	That is, $S_k$ represents all the points in $\L_n$ that can be
	reached from $o$ by at most $k$ edges of length $\tilde{r}^*_n$ over the
	corresponding rectilinear lattice. Denote by $\partial S_k$ the
	boundary of $S_k$, i.e., all the points that are exactly $k$
	edges from $o$. Let $A_k$ be the event that $\G^p_n$ contains a
	path from $o$ to some vertex in $\S_k$.
	
	\begin{theorem}{(\cite[Theorem~5.4]{Gri99})}
		If $p<p^*$, there exists a constant $c>0$ such that
		$\Pr[A_k]<e^{-ck}$ for all $k$. 
	\end{theorem}
	
	This theorem implies that in the subcritical regime it is
	improbable that $\G_n^p$ has a path connecting two lattice points $x(v),x(v')$ as
	defined above, as such a path must have $\omega(1)$ edges.
	
	\subsection{Back to motion planning}
	To conclude this section, we define a sparsified version of
    \prm defined on the deterministic samples $\L_n$, and state its
    theoretical properties. The proof is omitted as it is almost
    identical to the one obtained in Section~\ref{sec:prm}, with the
    difference that it now uses theorems and lemmata presented earlier
    this section, rather than the ones from
    Section~\ref{sec:elements}.
	
	Recall that \prm uses the two connection radii $r_n,r_n^{st}$. For simplicity we fix $r_n=\tilde{r}^*_n$. 
	
	\begin{definition}\label{def:prm_deter}
		The sparsified \prm graph $\P_n^p$ is the union between
		$\G^p_n\cap \F$ and the supplementary edges
		\[\bigcup_{y\in \{s,t\}} \left\{(x,y) \middle| x\in \G^p_n\cap
		\B_{r_n^{st}}(y)\textup{ and } xy\subset \F\right\}. \]
	\end{definition}
	
	\begin{theorem}\label{thm:prm_discrete}
		Suppose that $(\F,s,t)$ is robustly feasible. Then there exists
		$p^*>0$ such that the following holds:
		\begin{enumerate}[i.]
			\item If $p<p^*$ and $r^{st}_n=\infty$ then \prm fails (to find a
			solution) \aas
			\item Suppose that
			$p>p^*$. There exists $\beta'_0>0$ such that for 
			$r^{st}_n = \frac{\beta'\log^{1/(d-1)}n}{n^{1/d}}$, where $\beta'>\beta'_0$, and any 
			$\eps>0$ the following holds \aas:
			\begin{enumerate}[1.]
				\item $\P^p_n$ contains a path $\pi_n \in \Pi_{s,t}^\F$
				with $\len(\pi_n)\leq(1+\eps)\xi'\cdot\lenopt$;
				\item If $\M$ is well behaved then, $\P^p_n$ contains a path
				$\pi'_n \in \Pi_{s,t}^\F$ with
				$\bot(\pi'_n,\M)\leq (1+\eps) \botopt$.
			\end{enumerate}
		\end{enumerate}
	\end{theorem}
\vspace{5pt}



\section{Probabilistic completeness of RRT with samples from a PPP}\label{sec:rrt_pc}
This section is devoted to the proof of Claim~\ref{claim:rrt}. 
A probabilistic completeness proof of RRT, for the setting of uniform random sampling, is given in~\cite[Section~3]{KSKBH18}.
We describe here how to modify this proof for samples from a PPP, which are of interest in the present paper. 

Note that \rrt is an incremental planning algorithm, which generates samples one after the other, and connects the
current sample to previous ones. 
Therefore, it should employ the incremental PPP sampling defined in Claim~\ref{clm:inc_ppp}.

We wish to apply the line of arguments as for the uniform-sampling case described in~\cite[Theorem~1]{KSKBH18}. 
This theorem assumes that a path~$\pi$ from start to target exists, whose clearance is $\delta>0$ and length is $L$. The path~$\pi$
is covered with a set of $m$ balls of radius $r$, where $m= L/r$ and $r$ is a constant depending on both the steering parameter $\eta$ of \rrt and the clearance $\delta$.
The proof describes the process as a Markov chain with transition probability $p$, 
where $p$ is the probability  to produce an \rrt vertex in the $(i+1)$st ball $\B_r(x_{i+1})$ given that there exists an \rrt vertex in the $i$th ball $\B_r(x_{i})$.
It then shows that the probability that \rrt will not produce samples within all $m$ balls decays to zero exponentially as the number of samples tends to infinity. 
To adapt the proof for samples from a PPP, it suffices to outline how  to compute the transition probability $p$, 

Recall that at each iteration $i$
we draw a sample $N_i\in N$, where $N$ is a Poisson random variable with mean $1$,
and define $\X_i=\{X_1,\ldots,X_{N_i}\}$ to be $N_i$ points chosen
independently and uniformly at random from~$[0,1]^d$.
Let $A$ denote the event that $N_i\geq 1$,
and let $B$ denote the event that at least one of $X_1,\ldots,X_{N_i}$ falls inside $\B_{r}(x_{i+1})$.
Since $p=\Pr[A\cap B] = \Pr[A]\cdot\Pr[B|A]$, we will bound $\Pr[A],\Pr[B|A]$.
Observe that \[\Pr[A] =  \Pr[N_i\geq1] = 1-\Pr[N_i = 0] = 1- e^{-1},\] and note that $\Pr[B|A]\geq |\B_{r}|/|\F|$.	
Therefore, \[p = \Pr[A\cap B] \geq (1- e^{-1})\cdot|\B_{r}|/|\F|.\] 

The rest of the proof of Claim~\ref{claim:rrt} follows the same lines as~\cite[Theorem~1]{KSKBH18}.

    \section{Critical values}\label{sec:critical_values}
    In this section we provide for reference a list of (estimates) of the values  $\gamma^*$ and $p^*$.

	\subsection{Continuum percolation}
	Table~\ref{tbl:continuum} presents an estimate of $\gamma^*$ such that
	$\G_n$ percolates only if $r_n>\gamma^*n^{-1/d}$, for
	$2\leq d\leq 11$. It is directly derived
	from~\cite[Table~I]{TorJia12}, where an estimate of the critical
	node degree $\Delta(n,r_n)$ is derived. We mention that for $d$ large
	enough $\gamma^*\sim \frac{1}{2b_d^{1/d}}$~\citep{TorJia12}, where
	$b_d$ is volume of the Lebesgue measure of the unit ball
	in~$\dR^d$.
	
	\begin{table}[h]
		\centering
		\begin{tabular}{l l}
			\hline \hline
			$d$ & $\gamma^*$ \\
			\hline
			$2$ & $0.5992373341(3)$ \\
			$3$ & $0.4341989179(2)$ \\
			$4$ & $0.4031827664(5)$ \\ 
			$5$ & $0.4007817822(7)$ \\ 
			$6$ & $0.4067135508(5)$ \\ 
			$7$ & $0.4178609367(3)$ \\
			$8$ & $0.4317877097(6)$ \\ 
			$9$ & $0.447061366(4)$ \\
			$10$ & $0.462335684(3)$ \\ 
			$11$ & $0.4773913785(3)$ \\
			\hline \hline
		\end{tabular}
		\caption{Values of  $\gamma^*$ in continuum
			percolation.}
          \label{tbl:continuum}
	\end{table}

	\subsection{Lattice percolation}
	Table~\ref{tbl:lattice} presents the value of $p^*$ for $d=2$, the
	best known estimates for $4\leq d\leq 13$. Numbers in round
	brackets are single standard deviations.
	
	\begin{table}
		\centering
		\begin{tabular}{l l}
			\hline \hline
			$d$ & $p^*$ \\
			\hline
			$2$ & $0.5$ \\
			$3$ & $0.24881182(10)$\\
			$4$ & $0.1601314(13)$\\
			$5$ & $0.118172(1)$\\
			$6$ & $0.0942019(6)$\\
			$7$ & $0.0786752(3)$\\
			$8$ & $0.06770839(7)$\\
			$9$ & $0.05949601(5)$\\
			$10$ & $0.05309258(4)$\\
			$11$ & $0.04794969(1)$\\
			$12$ & $0.04372386(1)$\\ 
			$13$ & $0.04018762(1)$\\
			\hline \hline
		\end{tabular}
		\caption{Values of $p^*$ in lattice
			percolation. The case for $d=2$ is due to~\cite{FraMee08}; $d=3$ is due to~\cite{WanETAL13}; $d\in \{4,\ldots, 13\}$ is due to~\cite{Gra03}.}
          \label{tbl:lattice}
	\end{table}

\bibliographystyle{plainnat}
\bibliography{bibliography.bib}
	
\end{document}